\documentclass[a4paper,12pt]{article}%
\usepackage{amssymb}
\usepackage{natbib}
\usepackage{algorithmic}
\usepackage{amsfonts}
\usepackage[noblocks]{authblk}
\usepackage{amsmath}
\usepackage[nohead]{geometry}
\usepackage[singlespacing]{setspace}
\usepackage[bottom]{footmisc}
\usepackage{indentfirst}
\usepackage{endnotes}
\usepackage{graphicx}%
\usepackage{rotating}
\usepackage{amsthm}
\usepackage{subcaption}
\usepackage{booktabs}
\usepackage{tikz}
\usepackage{algorithm}
\usepackage{comment}
\usepackage{array}
\newcommand{\be} {\begin{eqnarray*}}
\newcommand{\ee} {\end{eqnarray*}}

\newcommand{\argmin}{\mathop{\rm argmin~}}

\newcommand{\bfz}{{\bf z}}

\newtheorem{theorem}{Theorem}[section]
\newtheorem{lemma}[theorem]{Lemma}

\newtheorem{rem}{Remark}[section]

\newtheorem{definition}{Definition}[section]

\small
\usepackage[colorlinks,bookmarksopen,bookmarksnumbered,citecolor=blue,urlcolor=blue,linkcolor=blue]{hyperref}

\begin{document}
\title{\ Sequential Gradient Descent and Quasi-Newton's Method for Change-Point Analysis}
\author{Xianyang Zhang\footnote{Address correspondence to Xianyang Zhang (zhangxiany@stat.tamu.edu).}}
\author{Trisha Dawn}
\affil{Texas A\&M University}
\date{\normalsize }
\maketitle

\singlespacing

\textbf{Abstract.} One common approach to detecting change-points is minimizing a cost function over possible numbers and locations of change-points. The framework includes several well-established procedures, such as the penalized likelihood and minimum description length. Such an approach requires finding the cost value repeatedly over different segments of the data set, which can be time-consuming when (i) the data sequence is long and (ii) obtaining the cost value involves solving a non-trivial optimization problem. This paper introduces a new sequential method (SE) that can be coupled with gradient descent (SeGD) and quasi-Newton's method (SeN) to find the cost value effectively. The core idea is to update the cost value using the information from previous steps without re-optimizing the objective function. The new method is applied to change-point detection in generalized linear models and penalized regression. Numerical studies show that the new approach can be orders of magnitude faster than the Pruned Exact Linear Time (PELT) method without sacrificing estimation accuracy. 

\section{Introduction}
Change-point analysis is concerned with detecting and locating structure breaks in the underlying model of a data sequence. The first work on change point analysis goes back to the 1950s, where the goal was to locate a shift in the mean of an independent and identically distributed Gaussian sequence for industrial quality control purposes \citep{page1954continuous,page1955test}. 
Since then,  change-point analysis has generated important activity in statistics and various application settings such as signal processing, climate science, economics, financial analysis, medical science, and bioinformatics. 
We refer the readers to \cite{brodsky1993nonparametric,csorgo1997limit,tartakovsky2014sequential} for book-length treatments and \cite{aue2013structural,niu2016multiple,aminikhanghahi2017survey,truong2020selective,liu2021high} for reviews on this subject.

There are two main branches of change-point detection methods: online methods that aim to detect changes as early as they occur in an online setting and offline methods that retrospectively detect changes when all samples have been observed. The focus of this paper is on the offline setting. A typical offline change-point detection method involves three major components: the cost function, the search method, and the penalty/constraint \citep{truong2020selective}. The choice of the cost function and search method has a crucial impact on the method's computational complexity. As increasingly larger data sets are being collected in modern applications, there is an urgent need to develop more efficient algorithms to handle such big data sets. Examples include testing the structure breaks for genetics data and detecting changes in the volatility of big financial data.

One popular way to tackle the change-point detection problem is to cast it into a model-selection problem by solving a penalized optimization problem over possible numbers and locations of change-points. The framework includes several well-established procedures, such as the penalized likelihood and minimum description length. The corresponding optimization can be solved exactly using dynamic programming \citep{auger1989algorithms,jackson2005algorithm} whose computational cost is $\sum^{T}_{t=1}\sum_{s=1}^{t}q(s)$, where $T$ is the number of data points and $q(s)$ denotes the time
complexity for calculating the cost function value based on $s$ data points. \cite{killick2012optimal} introduced the pruned exact linear time (PELT) algorithm with a pruning step in dynamic programming. PELT reduces the computational cost without affecting the exactness of the resulting segmentation.
\cite{rigaill2010pruned} proposed an alternative pruned dynamic programming algorithm with the
aim of reducing the computational effort.
However, in the worst case scenario, the computational cost of dynamic programming coupled with the above pruning strategies remains the order of $O(\sum^{T}_{t=1}\sum_{s=1}^{t}q(s))$.

Unlike the pruning strategy, this paper aims to improve the computational efficiency of dynamic programming from a different perspective. We focus on the class of problems where the cost function involves solving a non-trivial optimization problem without a closed-form solution. Dynamic programming requires repeatedly solving the optimization over different data sequence segments, which can be very time-consuming for big data. This paper makes the following contributions to address the issue.
\begin{enumerate}
    \item A new sequential updating method (SE) that can be coupled with the gradient descent (SeGD) and quasi-Newton's method (SeN) is proposed to update the parameter estimate and the cost value in dynamic programming. The new strategy avoids repeatedly optimizing the objective function based on each data segment. It thus significantly improves the computational efficiency of the vanilla PELT, especially when the cost function involves solving a non-trivial optimization problem without a closed-form solution. Though our algorithm is no longer exact, numerical studies suggest that the new method achieves almost the same estimation accuracy as PELT does.
    
    \item SeGD is related to the stochastic gradient descent (SGD) without-replacement sampling \citep{shamir2016without,nagaraj2019sgd,rajput2020closing}. The main difference is that our update is along the time order of the data points, and hence no sampling or additional randomness is introduced. Using some techniques from SGD and transductive learning theory, we obtain the convergence rate of the approximate cost value derived from the algorithm to the true cost value.
    
    \item The proposed method applies to a broad class of statistical models, such as parametric likelihood models, generalized linear models, nonparametric models, and penalized regression. 
\end{enumerate}

Finally, we mention two other routes to reduce the computational complexity in change-point analysis. The first one is to relax the $l_0$ penalty on the number of parameters to an $l_1$ penalty (such as the total variation penalty) on the parameters to encourage a piece-wise constant solution. The resulting convex optimization problem can be solved in nearly linear time \citep{harchaoui2010multiple}. In contrast, our method directly tackles the problem with the $l_0$ penalty. The second approach includes different approximation schemes, including window-based methods, binary segmentation and its variants \citep{vost1981,fryzlewicz2014wild}, and bottom-up segmentation \citep{ keogh2001online}. These methods are usually quite efficient and can be combined with various test statistics though they only provide approximate solutions. Our method can be regarded as a new approximation scheme for the $l_0$ penalization problem.

\begin{table}
\centering
{\renewcommand{\arraystretch}{1.5}
\begin{tabular}{l l l l}
& Dynamic programming & PELT & SE\\
\hline 
Time complexity & $\sum^{T}_{t=1}\sum_{s=1}^{t}q(s)$ & $\sum^{T}_{t=1}\sum_{s\in R_t}q(s)$  & $q_0\sum^{T}_{t=1}|R_t|$\\
\hline
\end{tabular}
}
\caption{Comparison of the computational complexity. Here $q(s)$ denotes the time complexity for calculating the cost value based on $s$ data points and $q_0$ is the time complexity for performing the one-step update described in Section \ref{sec:method}.
}\label{tb1}
\end{table}

The rest of the paper is organized as follows. In Section \ref{sec:review}, we briefly review the dynamic programming and the pruning scheme in change-point analysis. We describe the details of the Se algorithms in Section \ref{sec:method}, including the motivation, its application in generalized linear models, and an extension to handle the case where the cost value
involves solving a penalized optimization. We study the convergence property of the algorithm in Section \ref{sec:theory}.
Sections \ref{sec:sim} presents numerical results for synthesized and real data. Section \ref{sec:con} concludes.

\section{Dynamic Programming and Pruning}\label{sec:review}
\subsection{Dynamic programming}
Change-point analysis concerns the partition of a data set ordered by time (space or other variables) into piece-wise homogeneous segments such that each piece shares the same behavior. Specifically, we denote the data by $\bfz=(z_1,\dots,z_T)$. For $1\leq s\leq t\leq T$, let $\bfz_{s:t}=(z_s,\dots,z_t)$. If we assume that there are $k$ change-points in the data, then we can split the data into $k+1$ distinct segments. We let the location of the $j$th change-point be $\tau_j$ for $j=1,2,\dots,k,$ and set $\tau_0=0$ and $\tau_{k+1}=T.$  The $(j+1)$th segment contains the data $z_{\tau_{j}+1},\dots,z_{\tau_{j+1}}$ for 
$j=0,1,\dots,k$. We let $\boldsymbol{\tau}=(\tau_1,\dots,\tau_k)$ be the set of change-point locations. The problem we aim to address is to infer both the number of change points and their locations.

Throughout the discussions, we let $C(\bfz_{s+1:t})$ for $s<t$ denote the cost for a segment consisting of the data points $z_{s+1},\dots,z_t$. Of particular interest is the cost function defined as  
\begin{align}\label{eq}
C(\bfz_{s+1:t})=\min_{\theta\in\Theta}\sum^{t}_{i=s+1}l(z_i,\theta)
\end{align}
where $l(\cdot,\theta)$ is the individual cost parameterized by $\theta$ that belongs to a compact parameter space $\Theta\subset \mathbb{R}^d$. Examples include (i) $l(\cdot,\theta)$ is the negative log-likelihood of $z_i$; (2) $l(z_i,\theta)=L(f(x_i,\theta),y_i)$ with $z_i=(x_i,y_i)$, where $L$ is a loss function and $f(\cdot,\theta)$ is an unknown regression function parameterized by $\theta$. See more details and discussions in Section \ref{sec:model}. 

In this paper, we consider segmenting data by solving a penalized optimization problem. For $0\leq k\leq T-1$, define
\begin{align*}
C_{k,T}=\min_{\boldsymbol{\tau}}\sum^{k}_{j=0}C(\bfz_{\tau_j+1:\tau_{j+1}}).
\end{align*}
We estimate the number of change-points by minimizing
a linear combination of the cost value and a penalty function $f$, i.e.,
$$\min_k\{C_{k,T}+f(k,T)\}.$$
If the penalty function is linear in $k$ with $f(k,T)=\beta_T (k+1)$ for some $\beta_T>0,$ then we can write the objective function as
\begin{align*}
\min_k \left\{C_{k,T}+f(k,T)\right\}=\min_{k,\boldsymbol{\tau}}\sum^{k}_{j=0}\left\{C(\bfz_{\tau_j+1:\tau_{j+1}})+\beta_T\right\}.    
\end{align*}
One way to solve the penalized optimization problem is through the dynamic programming approach \citep{killick2012optimal,jackson2005algorithm}.
Consider segmenting the data $\bfz_{1:t}$. Denote $F(t)$ to be the minimum value of 
the penalized cost $\min_k \left\{C_{k,T}+f(k,T)\right\}$ for segmenting such data. We derive a recursion for $F(t)$ by conditioning on the last change-point location,
\begin{align}
F(t)=&\min_{k,\boldsymbol{\tau}}\sum^{k}_{j=0}\left\{C(\bfz_{\tau_j+1:\tau_{j+1}})+\beta_T\right\} \nonumber
\\=&\min_{k,\boldsymbol{\tau}}\left[\sum^{k-1}_{j=0}\left\{C(\bfz_{\tau_j+1:\tau_{j+1}})+\beta_T\right\}+C(\bfz_{\tau_k+1:t})+\beta_T\right] \nonumber
\\=&\min_{\tau}\left[\min_{\tilde{k},\boldsymbol{\tau}}\sum^{\tilde{k}}_{j=0}\left\{C(\bfz_{\tau_j+1:\tau_{j+1}})+\beta_T\right\}+C(\bfz_{\tau+1:t})+\beta_T\right] \nonumber
\\=&\min_{\tau}\left\{F(\tau)+C(\bfz_{\tau+1:t})+\beta_T\right\}, \label{eq1}
\end{align}
where $\tau_{k+1}=t$ in the first equation and $\tau_{\tilde{k}+1}=\tau$ in the third equation. 
The segmentations can be recovered by taking the argument $\tau$ which minimizes (\ref{eq1}), i.e.,
\begin{align}\label{eq-taustar}
\tau^*=\argmin_{0\leq \tau<t}\left\{F(\tau)+C(\bfz_{\tau+1:t})+\beta_T\right\},    
\end{align}
which gives the optimal location of the last change-point in the segmentation of $\bfz_{1:t}$. The procedure is repeated until all the change-point locations are identified.

\begin{figure}[H]
\centering
  \centering
  \includegraphics[width=0.7\linewidth]{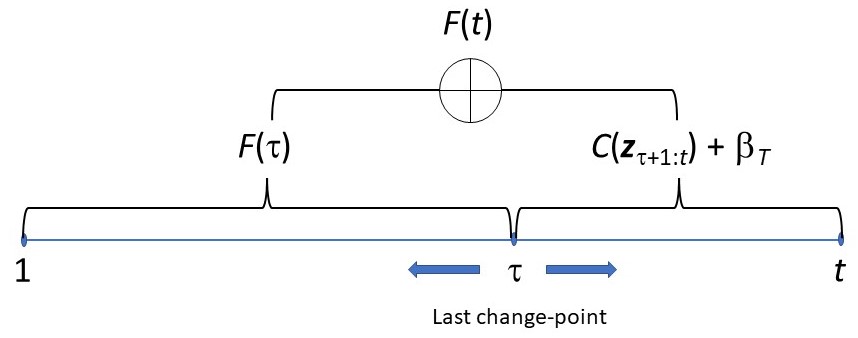}
\caption{Illustration of dynamic programming in change-point detection.}
\label{fig:dynam}
\end{figure}

\subsection{Pruning}
A popular way to increase the efficiency of dynamic programming is by pruning the candidate set for finding the last change-point in each iteration. 
For the cost function in (\ref{eq}), we have for any $\tau<t<T$, $C(\bfz_{\tau+1:t})+C(\bfz_{t+1:T})\leq C(\bfz_{\tau+1:T})$.
\cite{killick2012optimal} showed that for some $t>\tau$ if
\begin{align*}
F(\tau)+C(\bfz_{\tau+1:t})>F(t),
\end{align*}
then at any future point $t'>t$, $\tau$ can never be the optimal location of the most recent change-point prior to $t'$. Define a sequence of sets $\{R_t\}^{T}_{t=1}$ recursively as
\begin{align*}
R_t=\left\{\tau\in R_{t-1}\cup\{t-1\}: F(\tau)+C(\bfz_{\tau+1:t-1})\leq F(t-1)\right\}.    
\end{align*}
Then $F(t)$ can be computed as
\begin{align*}
F(t)=\min_{\tau\in R_{t}}\left\{F(\tau)+C(\bfz_{\tau+1:t})+\beta\right\}    
\end{align*}
and the minimizer $\tau^*$ in (\ref{eq-taustar}) belongs to $R_t$. This pruning technique forms the basis of the Pruned Exact Linear Time (PELT) algorithm. Under suitable conditions that allow the expected
number of change-points to increase linearly with $T$, \cite{killick2012optimal} showed that the expected computational cost  for PELT is bounded by $LT$ for some constant $L<\infty.$ 
In the worst case where no pruning occurs, the computational cost of PELT is the same as the vanilla dynamic programming.

\section{Methodology}\label{sec:method}
\subsection{Sequential algorithms}
For large-scale data, the computational cost of PELT can still be prohibitive due to the burden of repeatedly solving the optimization problem (\ref{eq}). For many statistical models, the time complexity for obtaining $C(\bfz_{s+1:t})$ is linear in the number of observations $t-s$. Therefore, in the worst-case scenario, the overall time complexity can be as high as $O(T^3)$. 
To alleviate the problem, we propose a fast algorithm by sequentially updating the cost function using a gradient-type method to reduce the computational cost while maintaining similar estimation accuracy. Instead of repeatedly solving the optimization problem to obtain the cost value for each data segment, we propose to update the cost value using the parameter estimates from the previous intervals. As the new method sequentially updates the parameter, we name it the sequential algorithm (SE).

We derive the algorithm here based on a heuristic argument. A rigorous justification for the convergence of the algorithm is given in Section \ref{sec:theory}. Suppose we have calculated $\hat{\theta}_{\tau+1:t-1}$, the approximation to $\tilde{\theta}_{\tau+1:t-1}$ that minimizes the cost function based on the data segment $\bfz_{\tau+1:t-1}$. We want to find the cost value for the next data segment $\bfz_{\tau+1:t}$, 
\begin{align}\label{eq-3}
C(\bfz_{\tau+1:t})=&\min_{\theta\in\Theta}\sum^{t}_{i=\tau+1}l(z_i,\theta),
\end{align}
where $\tau\geq 0$ and $t\leq T.$ Assume that $l(z,\theta)$ is twice differentiable in $\theta$.
Let $\tilde{\theta}_{\tau+1:t}$ be the minimizer of (\ref{eq-3}), which satisfies the first order condition (FOC)
\begin{align}
&\sum^{t}_{i=\tau+1}\nabla l(z_i,\tilde{\theta}_{\tau+1:t})=0. \label{s-2}
\end{align}
Taking a Taylor expansion around $\hat{\theta}_{\tau+1:t-1}$ in the FOC (\ref{s-2}), we obtain
\begin{align*}
0=&\sum^{t}_{i=\tau+1}\nabla l(z_i,\tilde{\theta}_{\tau+1:t})
\\ \approx &\sum^{t}_{i=\tau+1}\nabla l(z_i,\hat{\theta}_{\tau+1:t-1})+
\sum^{t-1}_{i=\tau+1}\nabla^2 l(z_i,\hat{\theta}_{\tau+1:t-1})(\tilde{\theta}_{\tau+1:t}-\hat{\theta}_{\tau+1:t-1})
\\ \approx &\nabla l(z_t,\hat{\theta}_{\tau+1:t-1})+
\sum^{t-1}_{i=\tau+1}\nabla^2 l(z_i,\hat{\theta}_{\tau+1:t-1})(\tilde{\theta}_{\tau+1:t}-\hat{\theta}_{\tau+1:t-1}),
\end{align*}
where $\sum^{t-1}_{i=\tau+1} l(z_i,\hat{\theta}_{\tau+1:t-1})\approx0$ as $\hat{\theta}_{\tau+1:t-1}$ is an approximate minimizer of $\sum^{t-1}_{i=\tau+1}l(z_i,\theta)$, and we drop the term $\nabla^2 l(z_t,\hat{\theta}_{\tau+1:t-1})$. Let $\mathcal{P}_{\Theta}(\theta)$ be the projection of any $\theta\in\mathbb{R}^d$ onto $\Theta$. 
The above observation motivates us to consider the following update
\begin{align*}
\hat{\theta}_{\tau+1:t}=\mathcal{P}_{\Theta}(\hat{\theta}_{\tau+1:t-1}-H_{\tau+1:t-1}^{-1}\nabla l(z_t,\hat{\theta}_{\tau+1:t-1})),  
\end{align*}
where $H_{\tau+1:t-1}$ is a preconditioning matrix that serves as a surrogate for the second order information
$\sum^{t-1}_{i=\tau+1}\nabla^2 l(z_i,\hat{\theta}_{\tau+1:t-1})$. When the second order information is available, we suggest update the preconditioning matrix through the iteration
\begin{align*}
H_{\tau+1:t} = H_{\tau+1:t-1}+  \nabla^2 l(z_{t},\hat{\theta}_{\tau+1:t}).
\end{align*}
Alternatively, by the idea of Fisher scoring, we can also update the preconditioning matrix by
\begin{align*}
H_{\tau+1:t} =H_{\tau+1:t-1}+ \mathcal{I}_{t}(\hat{\theta}_{\tau+1:t}),  
\end{align*}
where $\mathcal{I}_{t}(\theta)=E[\nabla^2 l(z_{t},\theta)|x_t]$ with $x_t$ being a subvector of $z_t$ such as covariates in the regression setting. Finally, we approximate $\tilde{\theta}_{\tau+1:t}$ by $(t-\tau)^{-1}\sum^{t}_{j=\tau+1}\hat{\theta}_{\tau+1:j}$ and the cost value $C(\bfz_{\tau+1:t})$ by
\begin{align*}
\widehat{C}(\bfz_{\tau+1:t})=\sum^{t}_{i=\tau+1}l\left(z_i,(t-\tau)^{-1}\sum^{t}_{j=\tau+1}\hat{\theta}_{\tau+1:j}\right).
\end{align*}
Algorithm \ref{alg1} below summarizes the details of the proposed algorithm.

\begin{algorithm}
\caption{Sequential Updating Algorithm}\label{alg1}
\begin{itemize}
    \item Input the data $\{z_i\}^{T}_{i=1}$, the individual cost function $l(\cdot,\theta)$ and the penalty constant $\beta$.
    \item Set $F(0)=-\beta$, $\mathcal{C}=\emptyset$ and $R_1=\{0\}$.
    \item Iterate for $t=1,2,\dots,T$:
    \begin{enumerate}
        \item Initialize $S_{t:t}=\hat{\theta}_{t:t}$ and $H_{t:t}$. For $\tau\in R_t\setminus\{t-1\}$, perform the update
        \begin{align*}
        &\hat{\theta}_{\tau+1:t}=\mathcal{P}_{\Theta}(\hat{\theta}_{\tau+1:t-1}-H_{\tau+1:t-1}^{-1}\nabla l(z_t,\hat{\theta}_{\tau+1:t-1})),\\
        &H_{\tau+1:t}=H_{\tau+1:t-1}+ \mathcal{A}_t(\hat{\theta}_{\tau+1:t}),\\
        &S_{\tau+1:t}=S_{\tau+1:t-1}+\hat{\theta}_{\tau+1:t}.
        \end{align*}
        \item For each $\tau\in R_t$, compute
        $$\widehat{C}(\bfz_{\tau+1:t})=\sum^{t}_{i=\tau+1}l\left(z_i,(t-\tau)^{-1}S_{\tau+1:t}\right).$$
        \item Calculate
        \begin{align*}
        & F(t)=\min_{\tau\in R_{t}}\left\{F(\tau)+\widehat{C}(\bfz_{\tau+1:t})+\beta\right\},\\
        & \tau^*=\argmin_{\tau\in R_{t}}\left\{F(\tau)+\widehat{C}(\bfz_{\tau+1:t})+\beta\right\}.
        \end{align*}
        \item Let $\mathcal{C}(t)=\{\mathcal{C}(\tau^*),\tau^*\}$.
        \item Set 
        $$R_{t+1}=\left\{\tau \in R_t\cup\{t\}:F(\tau)+C(\bfz_{\tau+1:t})\leq F(t)\right\}.$$
    \end{enumerate}
       \item Output $\mathcal{C}(T)$.
\end{itemize}
\end{algorithm}

\begin{rem}
{\rm 
When the second order information is available, we suggest setting $\mathcal{A}_t(\hat{\theta}_{\tau+1:t})=\mathcal{I}_t(\hat{\theta}_{\tau+1:t})$ in Algorithm \ref{alg1}, which leads to a type of sequential quasi-Newton's method taking into account the curvature information. On the other hand, our theory in Section \ref{sec:theory} allows $\mathcal{A}_t(\hat{\theta}_{\tau+1:t})=\mu I/2$ for some constant $\mu>0$ defined in Assumption A2.
} 
\end{rem}

\begin{rem}
{\rm 
When $\mathcal{A}_t(\hat{\theta}_{\tau+1:t})$ is rank one, i.e., $\mathcal{A}_t(\hat{\theta}_{\tau+1:t})=g_{\tau+1:t}g_{\tau+1:t}^\top$, the Sherman–Morrison formula suggests a recurve equation to update $H_{\tau+1:t}^{-1}$ directly:
\begin{align*}
H_{\tau+1:t}^{-1}=(H_{\tau+1:t-1}+ g_{\tau+1:t}g_{\tau+1:t}^\top)^{-1}=H_{\tau+1:t-1}^{-1}-\frac{H_{\tau+1:t-1}^{-1}  g_{\tau+1:t}g_{\tau+1:t}^\top H_{\tau+1:t-1}^{-1}}{1+g_{\tau+1:t}^\top  H_{\tau+1:t-1}^{-1}g_{\tau+1:t}}.
\end{align*}

}
\end{rem}

\begin{rem}
{\rm 
In order to speed up the optimization and avoid poor local minima, we can add a relatively large momentum term to the gradient, which leads to the following update:
$$\mathcal{P}_{\Theta}(\hat{\theta}_{\tau+1:t-1}-H_{\tau+1:t-1}^{-1}\nabla l(z_t,\hat{\theta}_{\tau+1:t-1})+a_{\tau+1:t-1}(\hat{\theta}_{\tau+1:t-1}-\hat{\theta}_{\tau+1:t-2})),$$
where $a_{\tau+1:t-1}$ represents the
momentum. 
}
\end{rem}

\begin{rem}
{\rm 
To initialize the estimate $\hat{\theta}_{t:t}$, we suggest dividing the data into a pre-determined number of segments and estimating the parameters using the data within each segment. We then set $\hat{\theta}_{t:t}$ to be the preliminary estimate using the data in the segment to which $t$ belongs.
}
\end{rem}

\begin{rem}
{\rm 
In practice, a post-processing step is recommended to remove the change-points in $\mathcal{C}(T)$ that are too close to the boundaries and merge those change-points that are too close to each other. 
}
\end{rem}

\begin{figure}[H]
\centering
  \centering
  \includegraphics[width=0.7\linewidth]{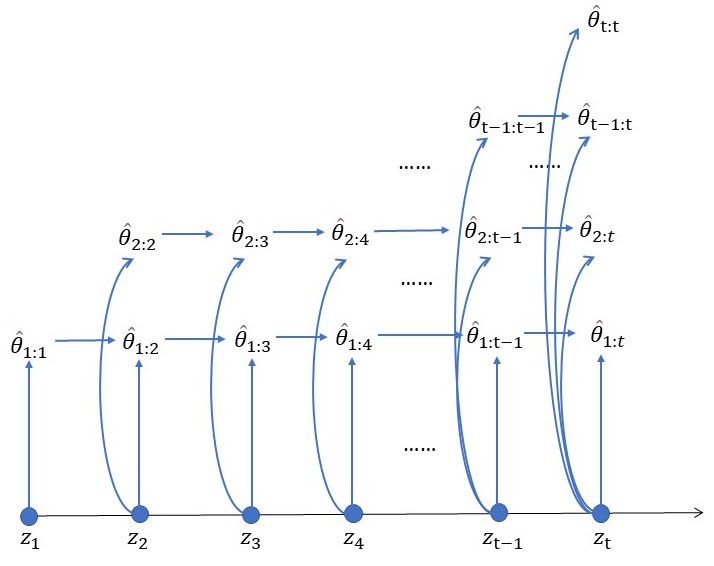}
\caption{Illustration of the updating scheme in the SE algorithms.}
\label{fig:0}
\end{figure}

\subsection{Generalized linear models}\label{sec:model}
As an illustration of our algorithm, we consider the change-point detection problem in the generalized linear models (GLM). In this case, $z_i$ contains a response $y_i$ and a set of predictors/covariates $x_i$. Suppose $y_i$ follows a
distribution in the canonical exponential family
\begin{align*}
f(y_i;\gamma_i,\phi)=\exp\left\{\frac{y_i\gamma_i-b(\gamma_i)}{w^{-1}\phi}+c(y_i,\phi)\right\},    
\end{align*}
where $\gamma_i$ is the canonical parameter, $\phi$ is the dispersion
parameter and $w$ is some known weight. The mean of $y_i$ is related to $x_i^\top\theta_i$ via $g(E[y_i])=g(\nabla b(\gamma_i))=x_i^\top\theta_i$, where $g$ is a known link function.
Suppose the observations from the time point $\tau+1$ to $t$ share the same parameter $\theta$, i.e., $\theta_i=\theta$ for $\tau+1\leq i\leq t.$
When $\phi$ is known, we let
$$l(z_i,\theta)=-\frac{y_i\gamma_i-b(\gamma_i)}{w^{-1}\phi}-c(y_i,\phi),\quad \tau+1\leq i\leq t.$$
Some algebra yields that
\begin{align*}
&\nabla l(z_i,\hat{\theta}_{\tau+1:t})=-\frac{a}{\phi v(\hat{\mu}_{i,\tau+1:t})g'(\hat{\mu}_{i,\tau+1:t})}(y_i-\hat{\mu}_{i,\tau+1:t})x_i,\\
&\mathcal{I}(x_i,\hat{\theta}_{\tau+1:t})=\frac{a}{\phi v(\hat{\mu}_{i,\tau+1:t})[g'(\hat{\mu}_{i,\tau+1:t})]^2}x_i x_i^\top,
\end{align*}
where $g(\hat{\mu}_{i,\tau+1:t})=x_i^\top \hat{\theta}_{\tau+1:t}$ and $v$ is related to the variance of $y_i$ through $\text{var}(y_i)=\phi w^{-1}v(\mu_i)$. 
In the cases of the logistic and Poisson regressions, we have $\omega=\phi=1$ and $g'(\mu)v(\mu)=1$. Hence for the logistic regression, 
\begin{align*}
&\nabla l(z_i,\hat{\theta}_{\tau+1:t})=-\left(y_i-\frac{e^{x_i^\top \hat{\theta}_{\tau+1:t}}}{1+e^{x_i^\top \hat{\theta}_{\tau+1:t}}}\right)x_i,\\
& \mathcal{I}(x_i,\hat{\theta}_{\tau+1:t})=(\hat{\mu}_{i,\tau+1:t})(1-\hat{\mu}_{i,\tau+1:t})x_ix_i^\top.
\end{align*}
While for the Poisson regression, 
\begin{align*}
&\nabla l(z_i,\hat{\theta}_{\tau+1:t})=-\left(y_i-e^{x_i^\top \hat{\theta}_{\tau+1:t}}\right)x_i,\\
& \mathcal{I}(x_i,\hat{\theta}_{\tau+1:t})=\hat{\mu}_{i,\tau+1:t}x_i x_i^\top.
\end{align*}
We shall investigate the performance of the corresponding algorithms in Section \ref{sec:sim}.

\subsection{Sequential proximal gradient descent}
In this section, we extend our algorithm to handle the case where the cost function value results from solving a penalized optimization problem. More precisely, let us consider
\begin{align}\label{eq-pen}
C(\bfz_{\tau+1:t})=\min_{\theta\in \Theta}\sum^{t}_{i=\tau+1}l(z_i,\theta)+\lambda_{\tau+1:t}\text{pen}(\theta),   
\end{align}
where the penalty term enforces a constraint on the parameter $\theta$ (e.g., the smoothness or sparsity constraint) and $\lambda_{\tau+1:t}>0$ is allowed to vary over data segments. Let
$$\text{Prox}(a;\lambda)=\argmin_{z}\frac{1}{2\lambda}\|z-a\|^2+\text{pen}(z)$$
be the proximal operator associated with the penalty term. For $\mathbf{a}=(a_1,\dots,a_d)$ with $a_i\neq 0$ and $\boldsymbol{\lambda}=(\lambda_1,\dots,\lambda_d)$ with $\lambda_i>0$, we write $\mathbf{a}^{-1}=(a^{-1}_1,\dots,a^{-1}_d)$ and $\text{Prox}(\mathbf{a};\boldsymbol{\lambda})=(\text{Prox}(a_1;\lambda_1),\dots,\text{Prox}(a_d;\lambda_d))$. We update the parameter estimate by
$\hat{\theta}_{\tau+1:t}=\mathcal{P}_\Theta(\breve{\theta}_{\tau+1:t})$ with
\begin{align*}
\breve{\theta}_{\tau+1:t}=\text{Prox}\left(\hat{\theta}_{\tau+1:t-1}-H_{\tau+1:t-1}^{-1}\nabla l(z_t,\hat{\theta}_{\tau+1:t-1});\lambda_{\tau+1:t} \|H_{\tau+1:t-1}\|_{2}^{-1}\right), 
\end{align*}
where $\|H\|_2$ denotes the spectral norm of $H$. An example here is the Lasso regression, where
$z_i=(x_i,y_i)\in\mathbb{R}\times \mathbb{R}^d$ and the objective function in (\ref{eq-pen}) can be written as 
$$\frac{1}{2}\sum^{t}_{i=\tau+1}\|y_i-x_i^\top \theta\|^2+\lambda_{\tau+1:t}\sum^{d}_{i=1}|\theta_i|$$
with $\theta=(\theta_1,\dots,\theta_d)$. In this case, $\text{Prox}(a;\lambda)=\text{sign}(a)\max(|a|-\lambda,0)$ is the soft thresholding operator.

\subsection{Choice of the penalty constant $\beta$}
Our method aims to solve the following $l_0$ penalized optimization problem approximately
\begin{align*}
\min_k\min_{\boldsymbol{\tau}}\left\{\sum^{k}_{j=0}\min_{\theta\in\Theta}\sum^{\tau_{j+1}}_{i=\tau_j+1}l(z_i,\theta) + f(k,T)\right\}
\end{align*}
 where we simultaneously optimize over the number of change-points $k$, the locations of change-points $\boldsymbol{\tau}$, and the parameters within each segment. With $k$ change-points that divide the data sequence into $k+1$ segments, the total number of parameters is $(k+1)d+k$, where $(k+1)d$ counts the number of parameters from the $k+1$ segments and $k$ corresponds to the number of change-points. We recommend setting $f(k,T)=\{(k+1)d+k\}\log(T)/2$ or equivalently 
$$\beta_T=(d+1)\log(T)/2,$$ 
which leads to the BIC criterion.

\section{Convergence Analysis}\label{sec:theory}
To understand why our method works, it is crucial to investigate how well the sequential gradient method can approximate the cost value for each data segment. To be clear, let us focus on the segment $\bfz_{1:n}$ with $1\leq n\leq T$.
Let 
$\hat{\theta}^*=\argmin_{\theta\in\Theta}F_n(\theta),$
where $F_n(\theta)=n^{-1}\sum^{n}_{i=1}l(z_i,\theta)$. Recall that given $\hat{\theta}_{1}$ (which only depends on $z_1$), we have the following updating scheme for finding an approximation to $\hat{\theta}^*$
\begin{align*}
\hat{\theta}_{1:t}=\mathcal{P}_{\Theta}(\hat{\theta}_{1:t-1}-H_{1:t-1}^{-1}\nabla l(z_t,\hat{\theta}_{t-1})), \quad 2\leq t\leq n,  
\end{align*}
where $H_{1:t-1}$ is a preconditioning matrix that only depends on $z_1,\dots,z_{t-1}$. 
Throughout this section, we write $\hat{\theta}_{1:t}=\hat{\theta}_t$, $H_{1:t}=H_t$ and $l(z_t,\theta)=l_t(\theta)$ for the ease of notation. Our analysis here focuses on the SeGD.

\begin{definition}[Strong convexity]
{\rm 
A differentiable function $F$ is said to be $\mu$-strongly convex, with $\mu>0$, if
and only if
$$F(\eta)\geq F(\theta)+\nabla F(\theta)^\top (\eta-\theta)+\frac{\mu}{2}\|\eta-\theta\|^2.$$ 
}
\end{definition}

\begin{definition}[Smoothness]
{\rm 
A differential function $F$ is said to be $L$-smooth if
$$|F(\eta)-F(\theta)-\nabla F(\theta)^\top (\eta-\theta)|\leq \frac{L}{2}\|\eta-\theta\|^2,$$
for any $\eta,\theta$ in the domain of $F$.
}
\end{definition}


We aim to quantify the difference $F_n(n^{-1}\sum^{n}_{t=1}\hat{\theta}_{t})-F_n(\hat{\theta}^*)$ and derive the convergence rate. To this end, we make
the following assumptions. 
\begin{enumerate}
    \item[A1.] There is an unknown change-point $1\leq \xi<T$.
    The first $\xi$ observations are assigned with their time locations through a random permutation $(\sigma(1),\dots,\sigma(\xi))$ while the last $n-\xi$ observations are assigned with the time locations through a random permutation $(\sigma(\xi+1),\dots,\sigma(T))$;  
    \item[A2.] $F_n$ is $\mu$-strongly convex;
    \item[A3.] $H_t=\eta_t I$ with $\eta_t=t\mu/2$;
    \item[A4.] $\sup_{1\leq i\leq n,\theta\in\Theta}\|\nabla l_i(\theta)\|\leq C$ for some constant $C>0$;
    \item[A5.] $\Theta$ is a compact set; 
    \item[A6.] $l_i(\theta)=f(x_i^\top\theta,y_i)+r(\theta)$, where $f(a,y)$ is $L_1$-Lipschitz and $L_2$-smooth in $a$ for any given $y$, $\|x_i\|$ is bounded almost surely, and $r$ is some fixed function. 
\end{enumerate}

\begin{theorem}\label{thm1}
{\rm 
Let $E^*$ be the expectation with respect to the random permutation $\sigma$ conditional on the observed data values $\{z_i\}^{n}_{i=1}$. Under Assumptions A1-A6, we have
\begin{align*}
E^*\left[F_n\left(n^{-1}\sum^{n}_{t=1}\hat{\theta}_{t}\right)-F_n(\hat{\theta}^*)\right]\leq \frac{c\log(n)}{n},
\end{align*}
for some constant $c$ that depends on $\mu, C,L_1$ and $L_2$.
}
\end{theorem}

\begin{rem}
{\rm 
In Assumption A.1, we assume that there is a single change point. Similar arguments can be used to handle the cases of no change point and multiple change-points.
}
\end{rem}

\begin{rem}
{\rm 
Assumptions A2, A4 and A6 are fulfilled for logistic models when $\|x_i\|$ is bounded almost surely and the smallest eigenvalue of $n^{-1}\sum^{n}_{t=1}x_t x_t^\top$ is bounded away from zero almost surely. We also remark that the same conclusion can be justified when Assumptions A2, A4, and A6 hold with probability converging to one by using the conditioning argument.
}
\end{rem}

\begin{rem}
{\rm 
Under Assumption A1, we can cast our method into a type of SGD without replacement sampling and employ the related techniques \citep{shamir2016without} to prove Theorem \ref{thm1}.
}
\end{rem}

\section{Numerical Studies}\label{sec:sim}
In this section, we apply the PELT method and the proposed SE method to several simulated data sets and a real data set to compare their estimation accuracy measured by the rand index and the computational time (in seconds).



\subsection{Generalized linear models}
We first consider the GLM with piecewise constant regression coefficients. The details for implementing SE under GLM has been described in Section \ref{sec:model}. 



\subsubsection{Logistic regression}\label{sec:logistic}
We first consider the logistic regression model:
\begin{align*}
y_i \sim \text{Bernoulli}\left(\frac{1}{1+e^{-x_i^T\theta_i}}\right),\quad x_i\sim N_d(0, \Sigma) \hspace{5pt} \text{with} \hspace{5pt} \Sigma = (0.9^{|i - j|})_{d \times d},\quad 1\leq i\leq T.
\end{align*}
Throughout the simulations, we set $T=1500$, $d\in\{1,3,5\}$ and vary the value of $\theta_i$ leading to different magnitudes of change. Let $\delta_d\in\mathbb{R}^d$ be the difference between the coefficients before and after a change-point. We choose $\delta_d$ such that $M(\delta_d):=\delta_d^\top \Sigma \delta_d\in\{0.36,0.81,1.96\}$ corresponding to small, medium and large magnitudes of change respectively.
We remark that the results are not sensitive to the specific choice of $\delta_d$ as long as $M(\delta_d)$ is held at the same level. For each configuration, we shall consider the number of change-points equal to $0,1,3$ and $5$. The detailed simulation settings for each case are given in Section \ref{sec:sim-set}. As seen from Figure \ref{fig:1}, SE achieves the same estimation accuracy in terms of the rand index as PELT does. SE could be around 350 times faster than PELT, making SE a highly scalable method in practice. For example, when the magnitude of change is small with three change-points for $d=5$, Se finished the analysis within $8.77$ seconds while it took $3133.58$ seconds for PELT to get the same result.


\subsubsection{Poisson regression}
Next, we consider the Poisson regression model given by
\begin{align*}
y_i \sim \text{Poisson}\left(e^{x_i^\top \theta_i}\right),\quad x_i\sim N_{d}(0, \Sigma) \hspace{5pt} \text{with} \hspace{5pt} \Sigma = (0.9^{|i - j|}))_{d \times d},\quad 1\leq i\leq T.
\end{align*}
The other simulation settings are the same as those for the logistic regression in Section \ref{sec:logistic} with the only exception of the $M(\delta_d)$ values. Here we set $M(\delta_d) \in \{0.01,0.05,0.2\}$, leading to small, medium, and large magnitudes of change, respectively. The way of generating the true regressions coefficients for each interval of observations partitioned by the true change-point locations is also the same as the logistic regression case; see Section \ref{sec:sim-set}. Figure \ref{fig:2} shows that SE performs as well as PELT in most cases at a much lower computational cost. For example, when there is only one change-point having a small magnitude of change in $d=3$, SE and PELT delivered the same rand index values with the computational time equal to 10.12 seconds and 5850 seconds, respectively, indicating that SE is around 578 times faster.

\subsection{Penalized linear regression}
We consider the linear model
\begin{align*}
y_i = x_i^\top\theta_i + \epsilon_i, \quad x_i\sim N_{d}(0, \Sigma) \hspace{5pt} \text{with} \hspace{5pt} \Sigma = 0.5I_{d \times d} \hspace{5pt} \text{and} \hspace{5pt} \epsilon_i\sim N(0,0.5),\quad 1\leq i\leq T.
\end{align*}
Set $T=1500, d=50$ and $s\in\{1,3,5\}$ where $s$ is the number of non-zero components of the $d$ dimensional regression coefficients $\theta_i$. The locations of the nonzero components are randomly selected. The magnitude of change is reflected by the difference between the $\theta_i$ values before and after the change-point(s). The values of the non-zero components of the regression coefficients $\theta_i$ within each odd-numbered segment partitioned by the change-points $\{\tau_i\}$ (i.e. $1\leq i \leq \tau_1$ and $\tau_j < i \leq \tau_{j+1}$ when $j$ is even) are set to be $1$. For the even-numbered segment (i.e. $\tau_j < i \leq \tau_{j+1}$ when $j$ is odd), the non-zero coefficients are generated from $N(1,\delta)$ with $\delta\in \{0.1,0.4,1\}$ corresponding to small, medium and large magnitudes of changes, respectively. Like the GLM simulation settings, we shall consider $0, 1, 3$, and $5$ change-points for different combinations of $s$ and magnitude of change. We consider the cost function 
\begin{align}\label{eq-pen2}
C(\bfz_{\tau+1:t})=\min_{\theta\in \Theta}\frac{1}{2}\sum^{t}_{i=\tau+1}\|y_i-x_i^\top\theta\|^2+\lambda_{\tau+1:t}\sum^{d}_{i=1}|\theta_i|,   
\end{align}
where $\lambda_{\tau+1:t}=\hat{\sigma}\sqrt{2\log(d)/(t-\tau)}$ with $\hat{\sigma}$ being a preliminary estimate of the noise level. In particular, we divide the data into ten segments, estimate the noise level within each segment using Lasso, and set $\hat{\sigma}$ to be the average of these estimates. We implement both PELT and SeGD (without the second order information) in this case.
As seen from Figures \ref{fig:3} and \ref{fig:5}, SeGD achieves competitive accuracy compared to PELT in most cases with lower computational cost. For instance, when $s=1$ and there is only one medium change-point, SeGD is about 8 times faster than PELT.

\subsection{A real data example}
We illustrate the method using a dataset from the immune correlates study of Maternal To Child Transmission (MTCT) of HIV-1 \citep{fong2015change}. The data set contains three variables: the 0/1 response $y_i$ indicating whether HIV transits from mother to child (79 HIV-transmitting mothers and 157 non-transmitting mothers, leading to $T=236$), childbirth delivery type $x_i$ (C-section/Vaginal), and the NAb score $z_i$ measuring the amount and breadth of neutralizing antibodies. 
We consider the following change point/threshold model: $y_i\sim \text{Bernoulli}(p_i)$ with
\begin{align*}
\log\left(\frac{p_i}{1-p_i}\right)=\tilde{x}_i^\top\beta(z_i),\quad \tilde{x}_i=(1,x_i)^\top, 
\end{align*}
where $\beta(z)=\sum^{k}_{j=0}\beta_j\mathbf{1}\{a_j\leq z<a_{j+1}\}$ with $\min_i z_i=a_0<a_1<a_2<\cdots<a_{k+1}=+\infty.$ In words, the regression coefficient is a piece-wise constant function of the NAb score.

To implement PELT and SE, we first sort the data in descending order according to the NAb score. Both methods find a single change that corresponds to the NAb score at 7.548556. SE finishes the analysis in 0.62 seconds, while PELT takes 22 seconds to get the same result. 

\section{Concluding Remarks}\label{sec:con}
To conclude, we point out three possible strategies
namely accelerated SeGD, multiple epochs and backward updating scheme to improve estimation accurate by speeding up the convergence in SeGD.
In Theorem \ref{thm1}, we have shown that the difference between $F_n\left(n^{-1}\sum^{n}_{t=1}\hat{\theta}_{t}\right)$ and the target cost value is of the order $O(\log(n)/n)$ with $n$ being the length of the segment. An interesting future direction is to develop an accelerated sequential gradient method to improve the convergence rate. Motivated by the accelerated SGD, we may consider the following update strategy:
\begin{align*}
&\beta_{\tau+1:t-1}=\alpha\hat{\theta}_{\tau+1:t-1}+(1-\alpha)v_{\tau+1:t-1},\\
&\hat{\theta}_{\tau+1:t}=\beta_{\tau+1:t-1}-H_{\tau+1:t-1}^{-1}\nabla L(f(x_t,\beta_{\tau+1:t-1}),y_t), \\
&\gamma_{\tau+1:t-1}=\beta_{\tau+1:t-1} + (1-\beta)v_{\tau+1:t-1}, \\
&v_{\tau+1:t}=\gamma_{\tau+1:t-1}-\zeta H_{\tau+1:t-1}^{-1}\nabla L(f(x_t,\beta_{\tau+1:t-1}),y_t),
\end{align*}
where we set $v_{\tau+1:\tau+1}=\hat{\theta}_{\tau+1:\tau+1}$ and $\alpha,\beta,\zeta$ are tuning parameters. An in-depth analysis of this algorithm is left for future research. Another way to improve the convergence is by using multiple epochs/passes over the data points in each segment. Algorithm \ref{alg1} only uses each data point once (one-pass) in updating the parameter estimates for a particular segment. Using multiple epochs has been shown to improve the rate of convergence \citep{nagaraj2019sgd}. In Section \ref{sec:ext}, we describe such an extension of our algorithm. Finally, one can introduce a backward updating scheme. Together with the forward scheme, we can update $\hat{\theta}_{a:t}$ using the estimates based on nearby segments $\hat{\theta}_{a+1:t}$ and $\hat{\theta}_{a:t-1}$, see Figure \ref{fig:segd2} for an illustration. 

\begin{figure}[H]
\centering
  \centering
  \includegraphics[width=0.7\linewidth]{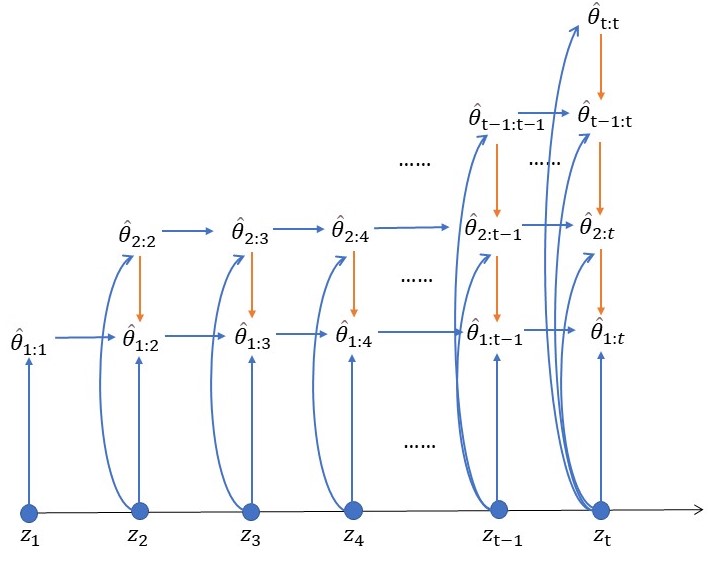}
\caption{SE with both forward (blue arrows) and backward (orange arrows) updating schemes.}
\label{fig:segd2}
\end{figure}

\bibliography{reference}

\begin{figure}[H]
\centering
\begin{subfigure}{0.9\textwidth}
  \centering
  \includegraphics[width=1\linewidth]{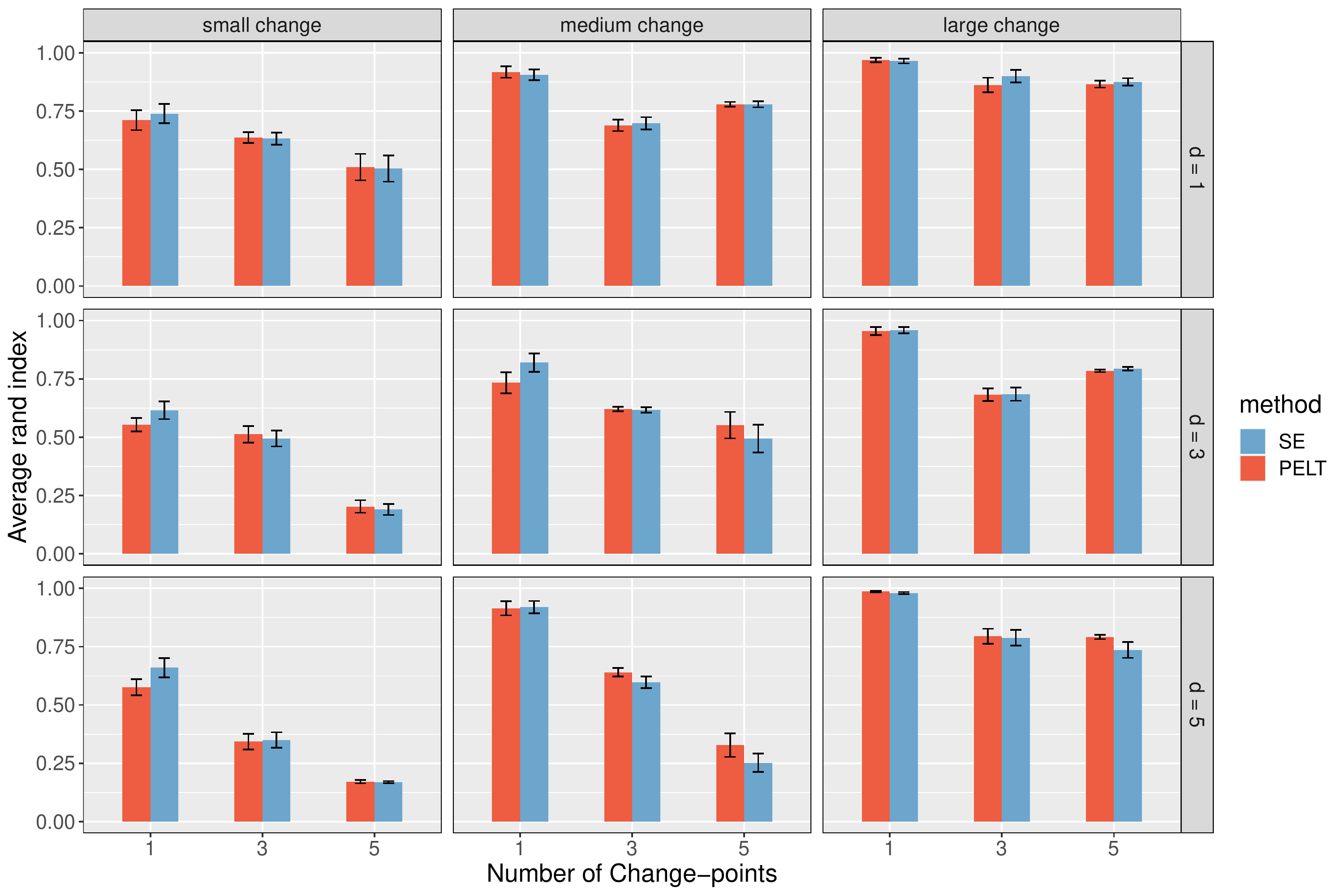}
  \caption{Average rand index}
  \label{fig:1a}
\end{subfigure}
\begin{subfigure}{0.9\textwidth}
  \centering
  \includegraphics[width=1\linewidth]{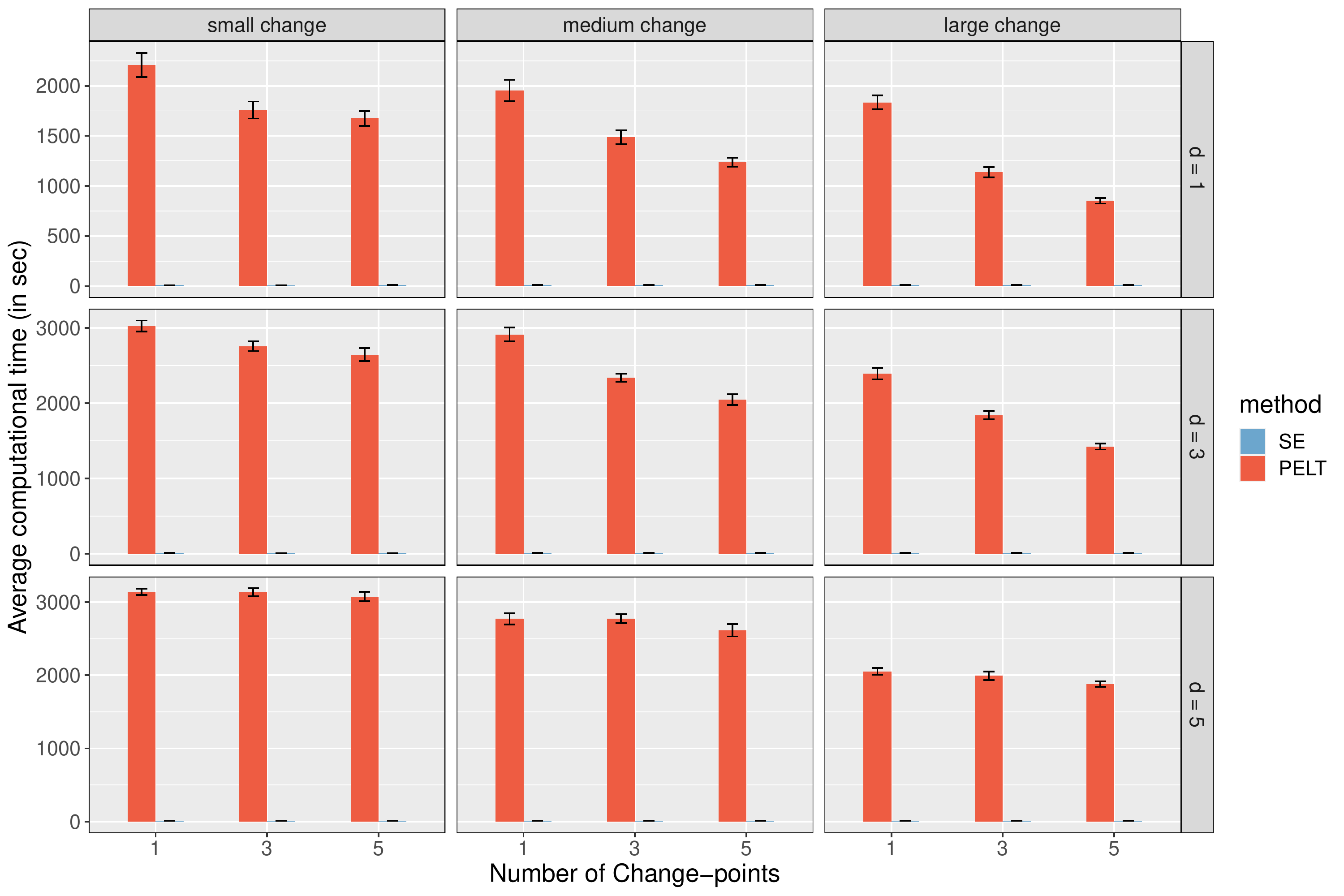}
  \caption{Computational time}
  \label{fig:1b}
\end{subfigure}
\caption{Average rand index and computational time for SE and PELT under logistic regression models with different number of change-points (1, 3, 5), magnitude of changes (small, medium and large) and dimension $d$ (1, 3, 5). Error bars represent the 95\% CIs ($\pm 2\times \text{standard error}$).}
\label{fig:1}
\end{figure}

\begin{figure}[H]
\centering
\begin{subfigure}{0.9\textwidth}
  \centering
  \includegraphics[width=1\linewidth]{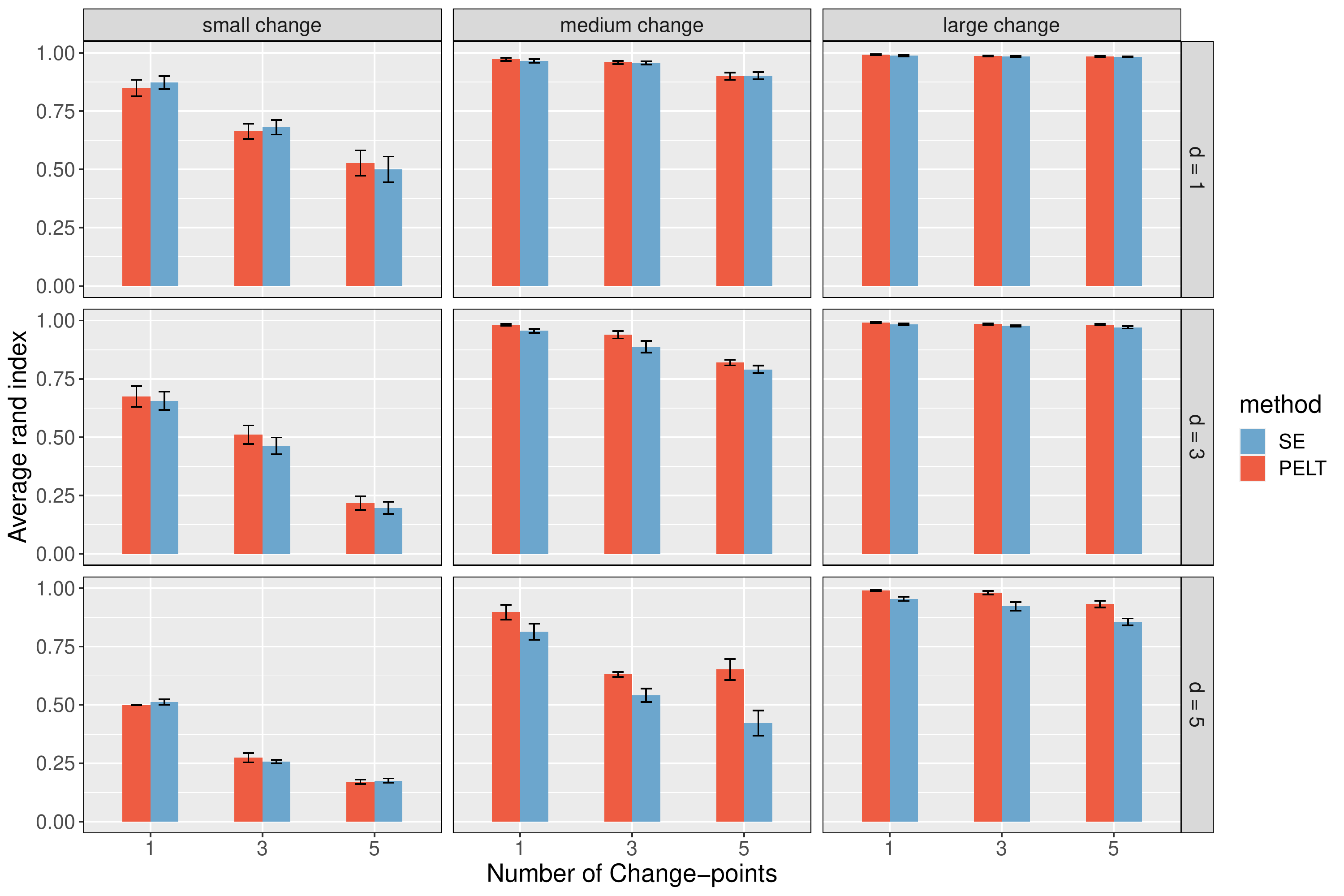}
  \caption{Average rand index}
  \label{fig:2a}
\end{subfigure}
\begin{subfigure}{0.9\textwidth}
  \centering
  \includegraphics[width=1\linewidth]{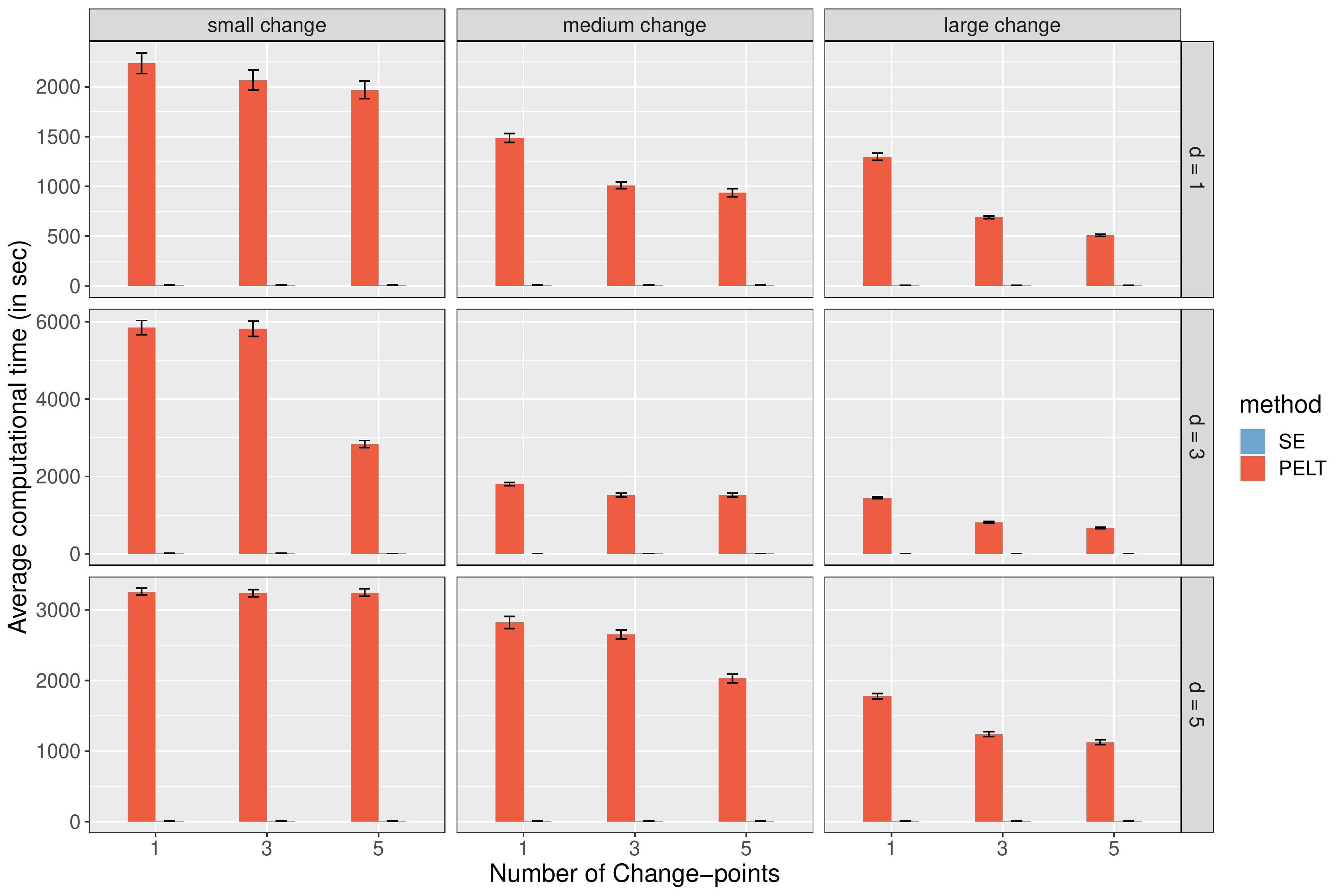}
  \caption{Computational time}
  \label{fig:2b}
\end{subfigure}
\caption{Average rand index and computational time for SE and PELT under Poisson regression models with different number of change-points (1, 3, 5), magnitude of changes (small, medium and large) and dimension $d$ (1, 3, 5). Error bars represent the 95\% CIs ($\pm 2\times \text{standard error}$).}
\label{fig:2}
\end{figure}

\begin{figure}[H]
\centering
\begin{subfigure}{0.9\textwidth}
  \centering
  \includegraphics[width=1\linewidth]{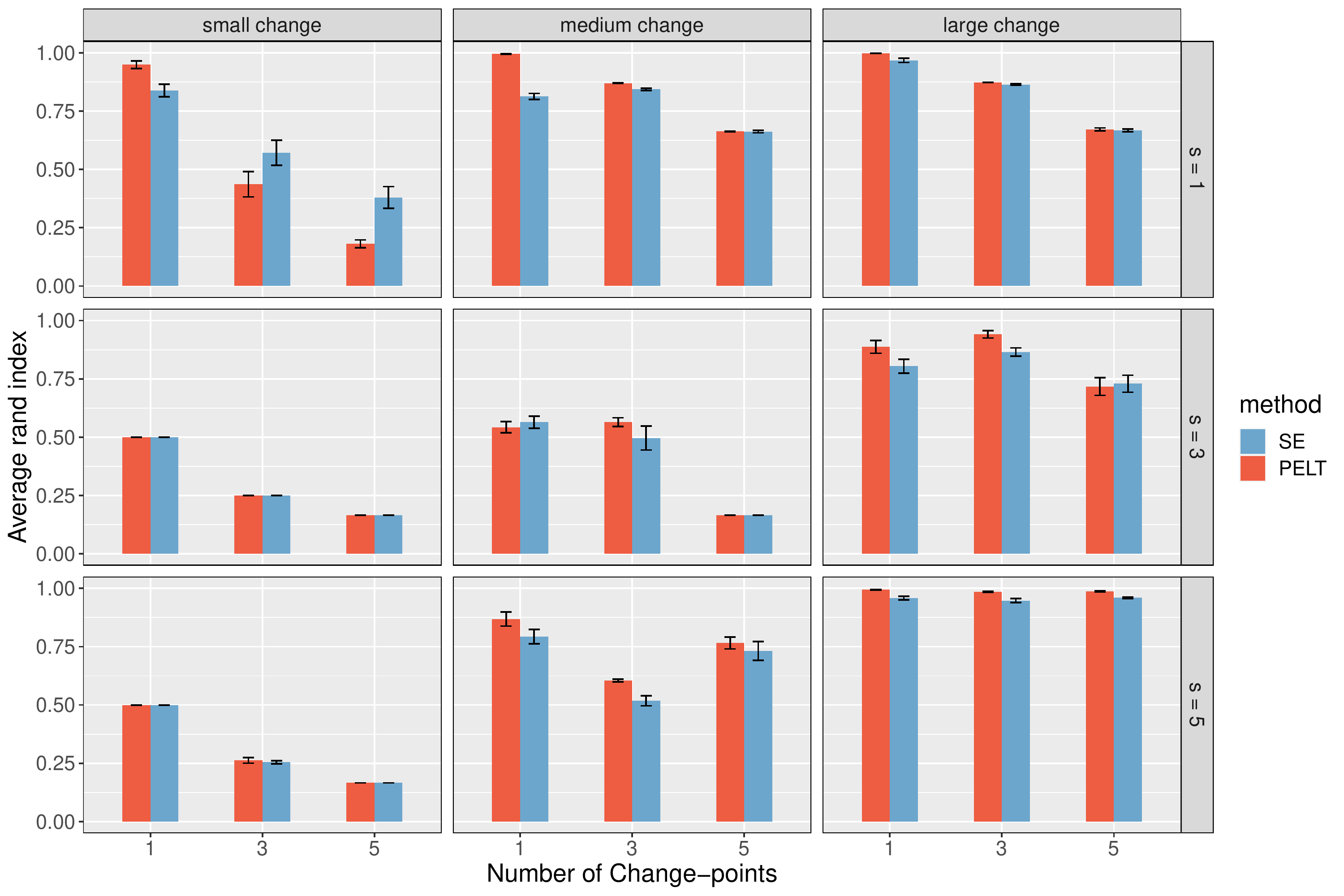}
  \caption{Average rand index}
  \label{fig:3a}
\end{subfigure}
\begin{subfigure}{0.9\textwidth}
  \centering
  \includegraphics[width=1\linewidth]{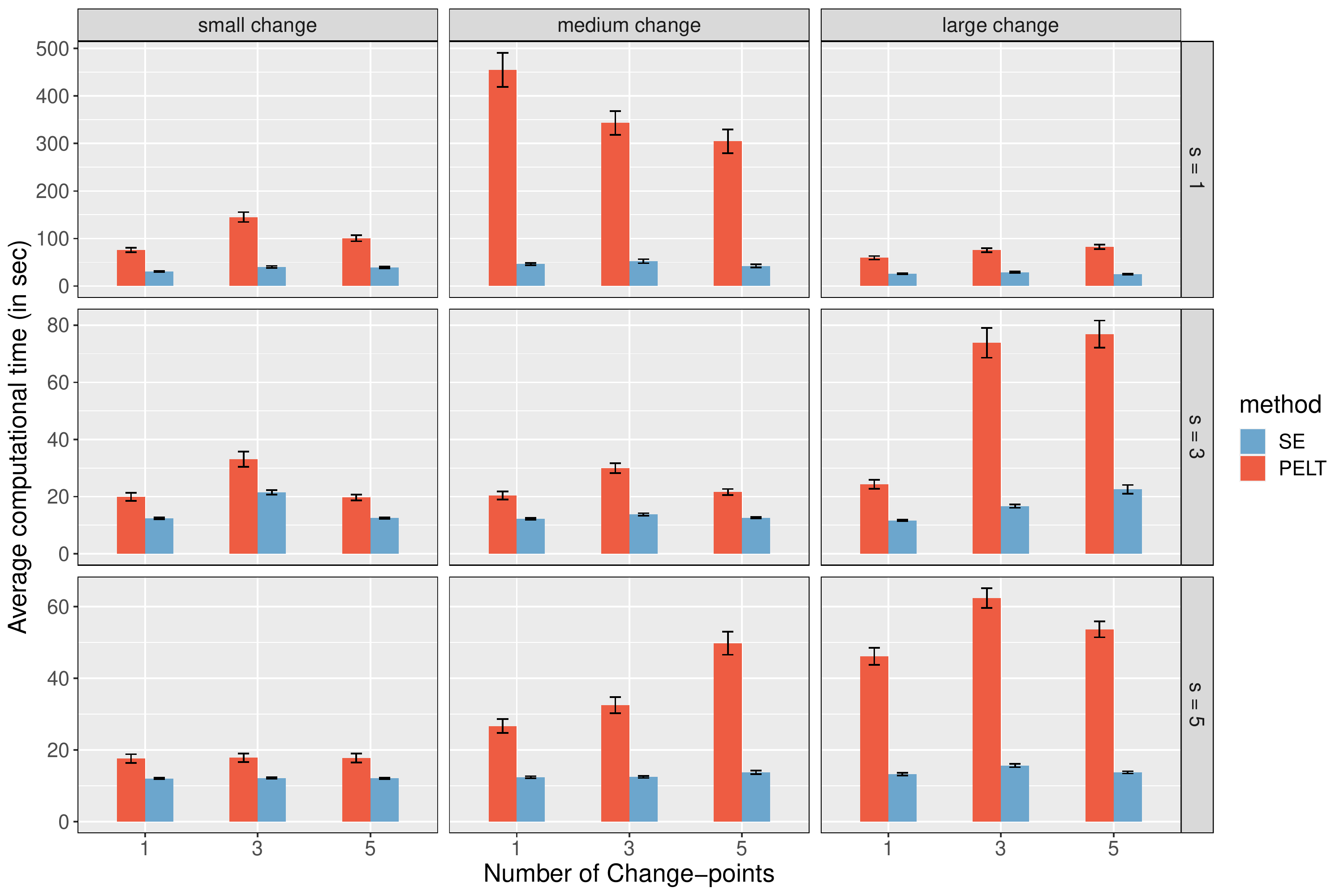}
  \caption{Computational time}
  \label{fig:3b}
\end{subfigure}
\caption{Average rand index and computational time for SE and PELT under penalized linear regression models with different number of change-points (1, 3, 5), magnitude of changes (small, medium and large) and the number of non-zero coefficients $s$ (1, 3, 5). Error bars represent the 95\% CIs ($\pm 2\times \text{standard error}$).}
\label{fig:3}
\end{figure}

\begin{figure}[H]
\centering
\begin{subfigure}{1\textwidth}
  \centering
  \includegraphics[height=6.5cm, width=12cm]{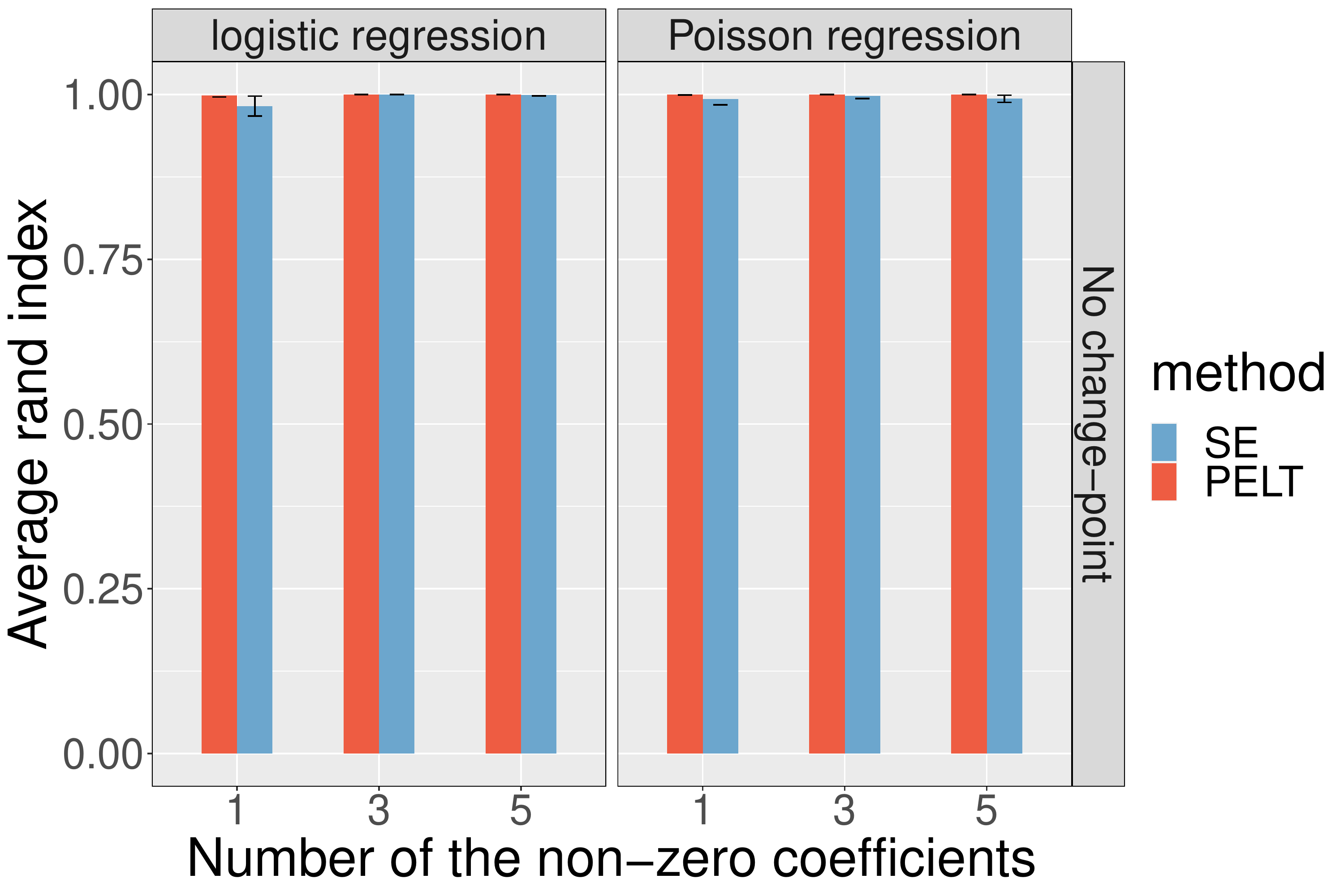}
  \caption{Average rand index}
  \label{fig:4a}
\end{subfigure}
\begin{subfigure}{1\textwidth}
  \centering
  \includegraphics[height=6.5cm, width=12cm]{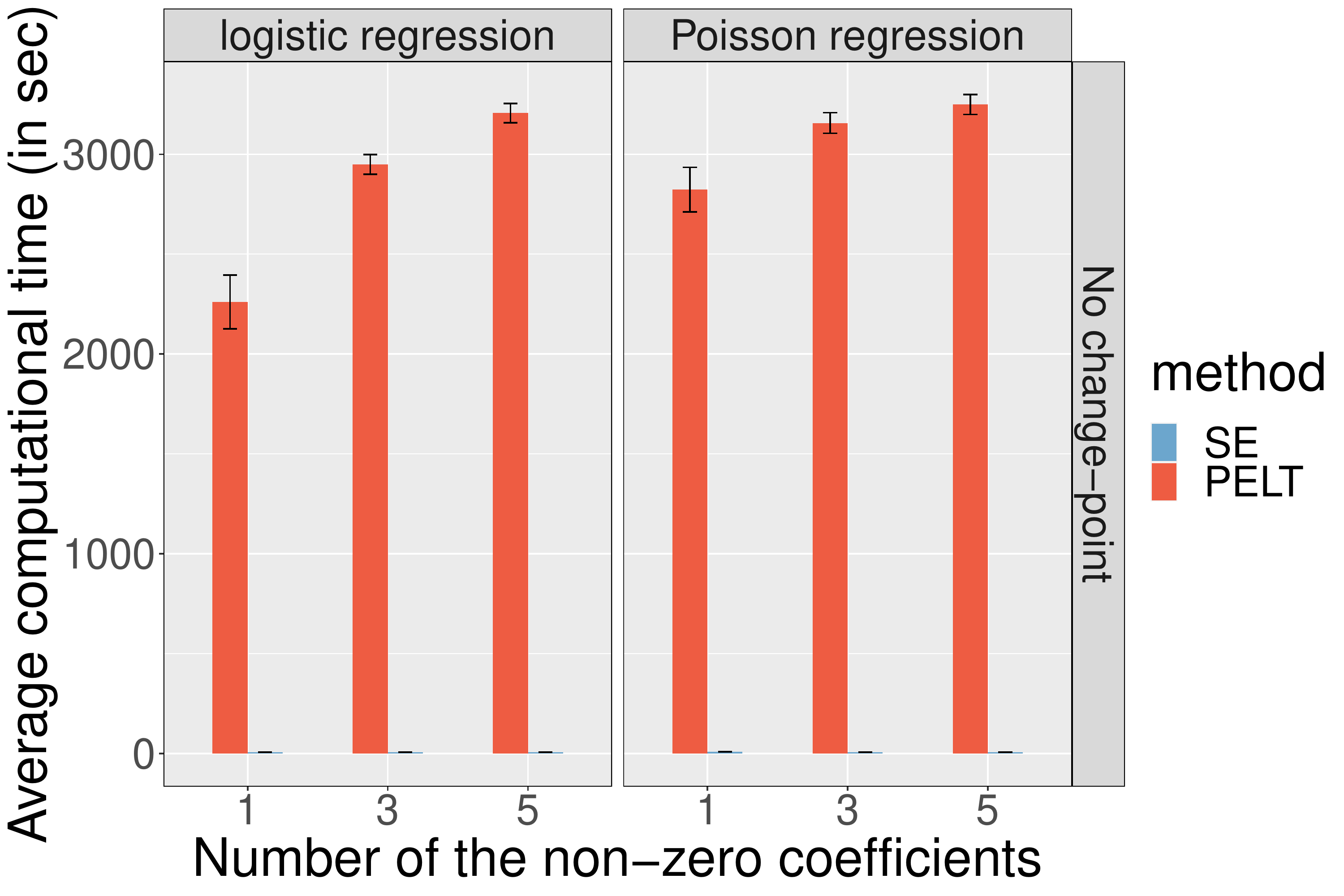}
  \caption{Computational time}
  \label{fig:4b}
\end{subfigure}
\caption{Average rand index and computational time for SE and PELT under logistic and Poisson regression models when there is no change-point. Error bars represent the 95\% CIs ($\pm 2\times \text{standard error}$).}
\label{fig:4}
\end{figure}

\begin{figure}[H]
\centering
  \centering
  \includegraphics[width=0.45\linewidth]{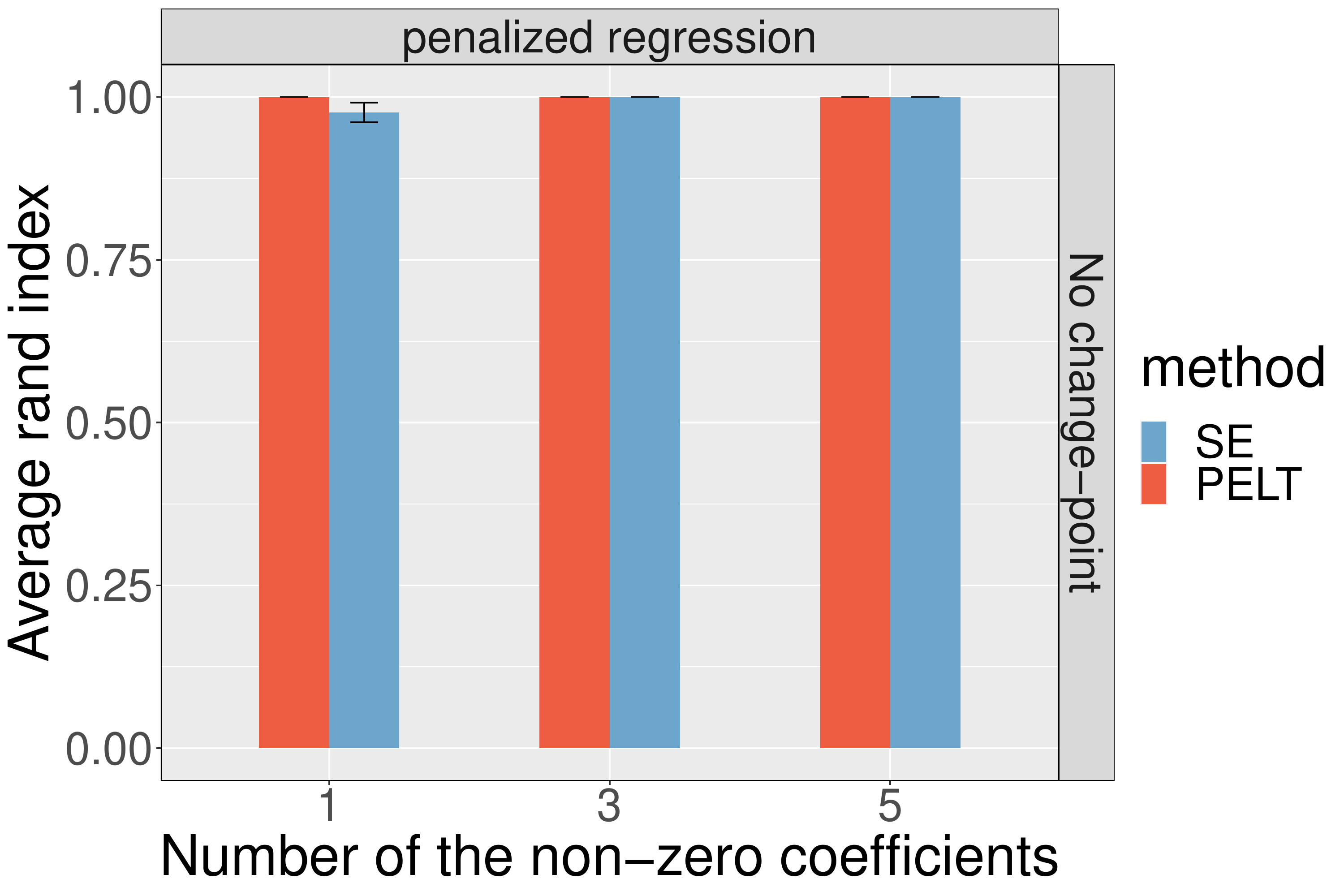}
  \includegraphics[width=0.45\linewidth]{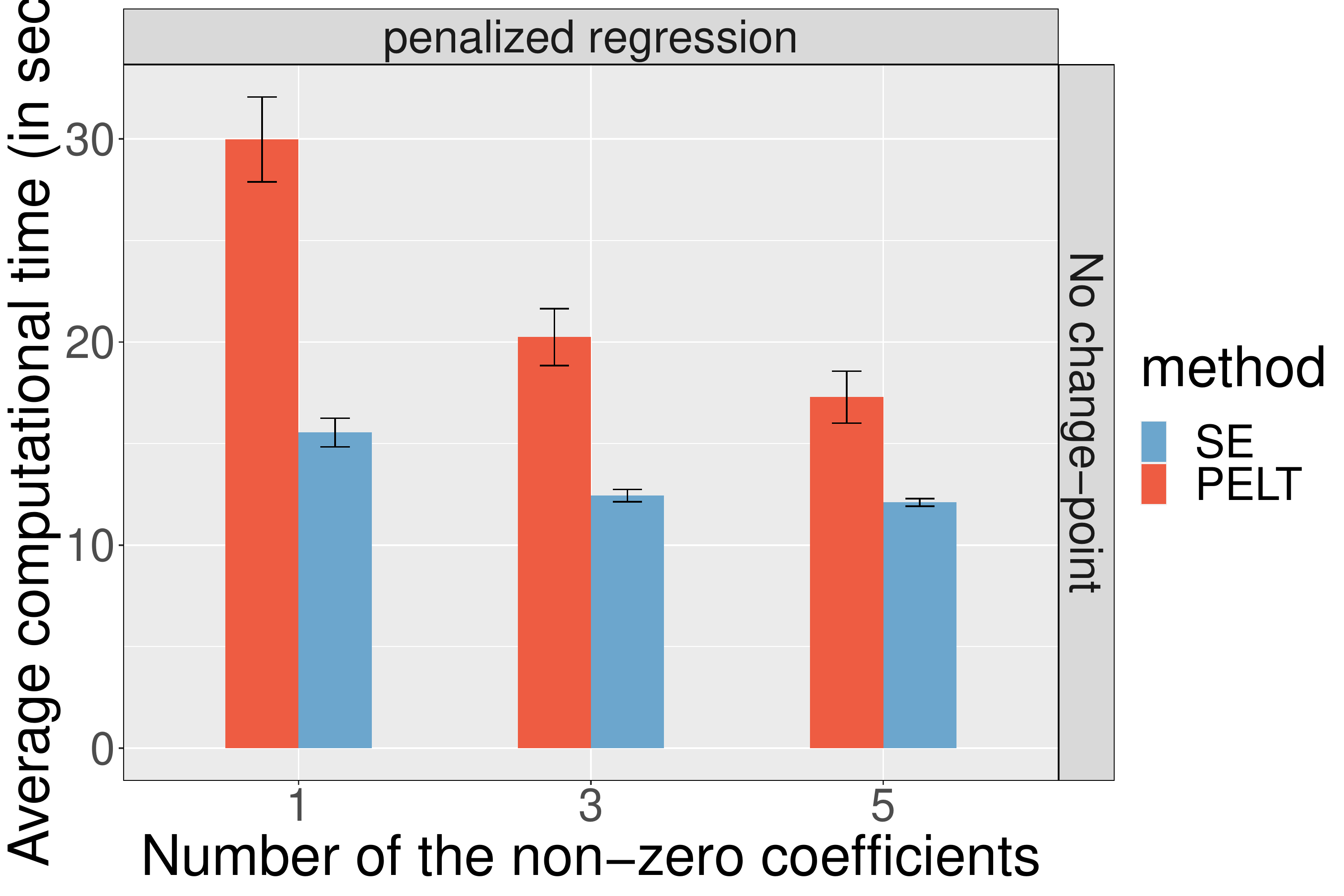}
\caption{Average rand index (left panel) and computational time (right panel) for SE and PELT under penalized regression model when there is no change-point. Error bars represent the 95\% CIs ($\pm 2\times \text{standard error}$).} 
\label{fig:5}
\end{figure}

\section{Appendix}
\subsection{Technical details}
\begin{proof}[Proof of Theorem \ref{thm1}]
Write $\sigma_{a:b}=(\sigma(a),\dots,\sigma(b))$ for $1\leq a\leq b\leq n$. By the definition of the algorithm, 
\begin{align*}
E^*[\|\hat{\theta}_t-\hat{\theta}^*\|^2]
=&E^*[\|\mathcal{P}_{\Theta}(\hat{\theta}_{t-1}-H_{t-1}^{-1}\nabla l_{\sigma_t}(\hat{\theta}_{t-1}))-\hat{\theta}^*\|^2]
\\ \leq & E^*[\|\hat{\theta}_{t-1}-H_{t-1}^{-1}\nabla l_{\sigma_t}(\hat{\theta}_{t-1})-\hat{\theta}^*\|^2]
\\ \leq & E^*[\|\hat{\theta}_{t-1}-\hat{\theta}^*\|^2]+\eta_{t-1}^{-2}C^2-2\eta_{t-1}^{-1}E^*[ \nabla l_{\sigma_t}(\hat{\theta}_{t-1})^\top (\hat{\theta}_{t-1}-\hat{\theta}^*)]
\\ = & E^*[\|\hat{\theta}_{t-1}-\hat{\theta}^*\|^2]+\eta_{t-1}^{-2}C^2-2\eta_{t-1}^{-1}E^*[\{\nabla l_{\sigma_t}(\hat{\theta}_{t-1})-\nabla F_n(\hat{\theta}_{t-1})\}^\top (\hat{\theta}_{t-1}-\hat{\theta}^*)]
\\ &-2\eta_{t-1}^{-1}E^*[ \nabla F_n(\hat{\theta}_{t-1})^\top (\hat{\theta}_{t-1}-\hat{\theta}^*)].
\end{align*}
By Assumption A2, namely the strong convexity, we have
\begin{align*}
\nabla F_n(\hat{\theta}_{t-1})^\top (\hat{\theta}_{t-1}-\hat{\theta}^*)\geq F_n(\hat{\theta}_{t-1})-F_n(\hat{\theta}^*)+\frac{\mu}{2}\|\hat{\theta}^*-\hat{\theta}_{t-1}\|^2,
\end{align*}
which implies that
\begin{align*}
E^*[\|\hat{\theta}_t-\hat{\theta}^*\|^2]
\leq  & E^*[\|\hat{\theta}_{t-1}-\hat{\theta}^*\|^2]+\eta_{t-1}^{-2}C^2-2\eta_{t-1}^{-1}E^*[\{\nabla l_{\sigma_t}(\hat{\theta}_{t-1})-\nabla F_n(\hat{\theta}_{t-1})\}^\top (\hat{\theta}_{t-1}-\hat{\theta}^*)]
\\ &-2\eta_{t-1}^{-1}E^*[F_n(\hat{\theta}_{t-1})-F_n(\hat{\theta}^*)]-\eta_{t-1}^{-1}\mu \|\hat{\theta}^*-\hat{\theta}_{t-1}\|^2.
\end{align*}
Re-arranging the terms, we get
\begin{equation}\label{eq-1}
\begin{split}
E^*[F_n(\hat{\theta}_{t-1})-F_n(\hat{\theta}^*)] 
\leq  & \left(\frac{\eta_{t-1}}{2}-\frac{\mu}{2}\right)E^*[\|\hat{\theta}_{t-1}-\hat{\theta}^*\|^2]-\frac{\eta_{t-1}}{2}E^*[\|\hat{\theta}_t-\hat{\theta}^*\|^2]+\frac{C^2}{2\eta_{t-1}}
\\&+E^*[\{\nabla F_n(\hat{\theta}_{t-1})-\nabla l_{\sigma_t}(\hat{\theta}_{t-1})\}^\top (\hat{\theta}_{t-1}-\hat{\theta}^*)].
\end{split}
\end{equation}
To deal with the last term on the RHS, we consider two cases, namely $t\leq \xi$ and $t>\xi$, separately. Let us first consider the case $t>\xi$. Conditional on $\sigma_{1:t-1}$ and the data values $\{z_i\}^{n}_{i=1}$, $\hat{\theta}_{t-1}$ is fixed. Note that $\hat{\theta}^*$ is independent of any permutation of the data. Moreover,  
$\sigma_t$ is uniformly distributed on $\{1,2,\dots,n\}\setminus\sigma_{1:t-1}=\{\sigma_{t},\dots,\sigma_n\}$. We also note that 
\begin{align*}
E^*[\nabla F_n(\hat{\theta}_{t-1})|\sigma_{1:t-1}]=&\frac{1}{n}\sum^{t-1}_{j=1}\nabla l_{\sigma_j}(\hat{\theta}_{t-1})+\frac{1}{n}\sum^{n}_{j=t}E^*[\nabla l_{\sigma_j}(\hat{\theta}_{t-1})|\sigma_{1:t-1}] 
\\=&\frac{1}{n}\sum^{t-1}_{j=1}\nabla l_{\sigma_j}(\hat{\theta}_{t-1})+\frac{1}{n}\sum^{n}_{j=t}\nabla l_{\sigma_j}(\hat{\theta}_{t-1})
\\=&\nabla F_n(\hat{\theta}_{t-1}).
\end{align*}
Using these facts, we have
\begin{align*}
&E^*[\{\nabla F_n(\hat{\theta}_{t-1})-\nabla l_{\sigma_t}(\hat{\theta}_{t-1})\}^\top (\hat{\theta}_{t-1}-\hat{\theta}^*)]
\\=&E^*[E^*[\{\nabla F_n(\hat{\theta}_{t-1})-\nabla l_{\sigma_t}(\hat{\theta}_{t-1})\}^\top (\hat{\theta}_{t-1}-\hat{\theta}^*)|\sigma_{1:t-1}]]
\\=&E^*\left[\left\{\nabla F_n(\hat{\theta}_{t-1})-(n-t+1)^{-1}\sum^{n}_{j=t}\nabla l_{\sigma_j}(\hat{\theta}_{t-1})\right\}^\top (\hat{\theta}_{t-1}-\hat{\theta}^*)\right]
\\=&\frac{t-1}{n}E^*\left[\left\{\nabla l_{ 1:t-1}(\hat{\theta}_{t-1})-\nabla l_{t:n}(\hat{\theta}_{t-1})\right\}^\top (\hat{\theta}_{t-1}-\hat{\theta}^*)\right]
\\=&\frac{t-1}{n}E^*\left[\|\hat{\theta}_{t-1}-\hat{\theta}^*\|\left\{\nabla l_{ 1:t-1}(\hat{\theta}_{t-1})-\nabla l_{t:n}(\hat{\theta}_{t-1})\right\}^\top \frac{\hat{\theta}_{t-1}-\hat{\theta}^*}{\|\hat{\theta}_{t-1}-\hat{\theta}^*\|}\right]
\\ \leq &\frac{t-1}{n}\sqrt{E^*\left[\|\hat{\theta}_{t-1}-\hat{\theta}^*\|^2\right]}
\sqrt{E^*\left[\left(\sup_{\theta\in\Theta}\left\{\nabla l_{ 1:t-1}(\theta)-\nabla l_{t:n}(\theta)\right\}^\top \frac{\theta-\hat{\theta}^*}{\|\theta-\hat{\theta}^*\|}\right)^2\right]}
\end{align*}
where we have defined $\nabla l_{a:b}(\theta)=\sum^{b}_{i=a}\nabla l_{\sigma_i}(\theta)/(b-a+1)$. 
Applying Lemma \ref{lem-1}, the above expression is at most
\begin{align*}
&\frac{C}{n}\sqrt{E^*\left[\|\hat{\theta}_{t-1}-\hat{\theta}^*\|^2\right]}
\left(
\sqrt{t-1}+\frac{t-1}{\sqrt{n-t+1}}\right)
\\ \leq & \frac{\mu}{4}E^*\left[\|\hat{\theta}_{t-1}-\hat{\theta}^*\|^2\right] + \frac{C^2}{\mu n^2}\left(
\sqrt{t-1}+\frac{t-1}{\sqrt{n-t+1}}\right)^2
\\ \leq & \frac{\mu}{4}E^*\left[\|\hat{\theta}_{t-1}-\hat{\theta}^*\|^2\right] + \frac{2 C^2}{\mu n^2}\left(
t-1+\frac{(t-1)^2}{n-t+1}\right),
\end{align*}
where the first inequality follows from the fact that $\sqrt{ab}\leq \mu a/4 +b/\mu$ and the second inequality is due to $(a+b)^2\leq 2a^2+2b^2.$

Next we consider the case where $t\leq \xi$. Conditional on $\sigma_{1:t-1}$, $\sigma_t$ is uniformly distributed on $\{1,2,\dots,\xi\}\setminus \sigma_{1:t-1}=\{\sigma_t,\dots,\sigma_\xi\}$. Similar arguments show that
\begin{align*}
&E^*[\{\nabla F_n(\hat{\theta}_{t-1})-\nabla l_{\sigma_t}(\hat{\theta}_{t-1})\}^\top (\hat{\theta}_{t-1}-\hat{\theta}^*)]
\\=&E^*\left[\left\{\nabla F_n(\hat{\theta}_{t-1})-(\xi-t+1)^{-1}\sum^{\xi}_{j=t}\nabla l_{\sigma_j}(\hat{\theta}_{t-1})\right\}^\top (\hat{\theta}_{t-1}-\hat{\theta}^*)\right]
\\=&\frac{n-\xi+t-1}{n}E^*\left[\|\hat{\theta}_{t-1}-\hat{\theta}^*\|\left\{\nabla l_{ -(t:\xi)}(\hat{\theta}_{t-1})-\nabla l_{t:\xi}(\hat{\theta}_{t-1})\right\}^\top \frac{\hat{\theta}_{t-1}-\hat{\theta}^*}{\|\hat{\theta}_{t-1}-\hat{\theta}^*\|}\right]
\\ \leq &\frac{n-\xi+t-1}{n}\sqrt{E^*\left[\|\hat{\theta}_{t-1}-\hat{\theta}^*\|^2\right]}
\sqrt{E^*\left[\left(\sup_{\theta\in\Theta}\left\{\nabla l_{ -(t:\xi)}(\theta)-\nabla l_{t:\xi}(\theta)\right\}^\top \frac{\theta-\hat{\theta}^*}{\|\theta-\hat{\theta}^*\|}\right)^2\right]},
\end{align*}
where $\nabla l_{ -(t:\xi)}(\theta)=(n-\xi+t-1)^{-1}\{\sum^{t-1}_{j=1}\nabla l_j(\theta)+\sum^{n}_{j=\xi+1}\nabla l_j(\theta)\}.$ As
\begin{align*}
\nabla l_{ -(t:\xi)}(\theta)-\nabla l_{t:\xi}(\theta)=w(\nabla l_{1:t-1}(\theta)-\nabla l_{t:\xi}(\theta)) 
+ (1-w)(\nabla l_{\xi+1:n}(\theta)-\nabla l_{t:\xi}(\theta))   
\end{align*}
with $w=(t-1)/(n-\xi+t-1)$, we have
\begin{align*}
& E^*\left[\left(\sup_{\theta\in\Theta}\left\{\nabla l_{t:\xi}(\theta)-\nabla l_{ -(t:\xi)}(\theta)\right\}^\top \frac{\theta-\hat{\theta}^*}{\|\theta-\hat{\theta}^*\|}\right)^2\right]
\\ \leq &2w^2 E^*\left[\left(\sup_{\theta\in\Theta}\left\{\nabla l_{t:\xi}(\theta)-\nabla l_{ 1:t-1}(\theta)\right\}^\top \frac{\theta-\hat{\theta}^*}{\|\theta-\hat{\theta}^*\|}\right)^2\right]
\\&+2(1-w)^2 E^*\left[\left(\sup_{\theta\in\Theta}\left\{\nabla l_{t:\xi}(\theta)-\nabla l_{ \xi+1:n}(\theta)\right\}^\top \frac{\theta-\hat{\theta}^*}{\|\theta-\hat{\theta}^*\|}\right)^2\right].
\end{align*}
Similar argument as before gives
\begin{align*}
&E^*[\{\nabla F_n(\hat{\theta}_{t-1})-\nabla l_{\sigma_t}(\hat{\theta}_{t-1})\}^\top (\hat{\theta}_{t-1}-\hat{\theta}^*)]
\\ \leq & \frac{\mu}{4}E^*\left[\|\hat{\theta}_{t-1}-\hat{\theta}^*\|^2\right] + \frac{4C^2}{\mu n^2}
\Bigg\{\left(
\frac{(t-1)^2}{\xi-t+1}+t-1\right)
+
\left(
\frac{(n-\xi)^2}{\xi-t+1}+n-\xi\right)\Bigg\}.
\end{align*}
Using the above bounds and averaging over $2,\dots,n+1$ of (\ref{eq-1}), we obtain
\begin{align*}
&E^*\left[n^{-1}\sum^{n}_{t=1}F_n(\hat{\theta}_{t})-F_n(\hat{\theta}^*)\right] \\
\leq  & \frac{1}{2n}\sum^{n}_{t=1}\left(\eta_{t}-\mu/2-\eta_{t-1}\right)E^*[\|\hat{\theta}_{t}-\hat{\theta}^*\|^2]+\frac{C^2}{2n}\sum^{n}_{t=1}\frac{1}{\eta_t} 
\\&+\frac{4 C^2}{\mu n^3}\sum^{\xi}_{t=2}\Bigg\{\left(
\frac{(t-1)^2}{\xi-t+1}+t-1\right)
+
\left(
\frac{(n-\xi)^2}{\xi-t+1}+n-\xi\right)\Bigg\}
\\&+\frac{4 C^2}{\mu n^3}\sum^{n+1}_{t=\xi+1}\left(
t-1+\frac{(t-1)^2}{n-t+1}\right),
\end{align*}
where $\eta_0=0$ and we have replaced the dummy variable $t-1$ with $t$ in the summation.
Note that
\begin{align*}
&\sum^{\xi}_{t=2}\Bigg\{\left(
\frac{(t-1)^2}{\xi-t+1}+t-1\right)
+
\left(
\frac{(n-\xi)^2}{\xi-t+1}+n-\xi\right)\Bigg\}+\sum^{n+1}_{t=\xi+1}\left(
t-1+\frac{(t-1)^2}{n-t+1}\right)
\\=&\sum^{n}_{t=1}t+\sum^{\xi-1}_{t=1}\Bigg\{
\frac{t^2}{\xi-t}
+
\frac{(n-\xi)^2}{\xi-t}+n-\xi\Bigg\}+\sum^{n}_{t=\xi}\frac{t^2}{n-t}
\\ \leq& C'n^2\log(n),
\end{align*}
for $C'>0$, where we have used the following facts
\begin{align*}
&\sum^{n}_{t=1}t=\frac{(n+1)n}{2}\leq C_1 n^2,\\
&\sum^{\xi-1}_{t=1}\frac{(n-\xi)^2}{\xi-t}\leq C_2n^2\log(n),\\
&\sum^{\xi-1}_{t=1}
\frac{t^2}{\xi-t}\leq (\xi-1)^2\sum^{\xi-1}_{t=1}\frac{1}{t}\leq C_3n^2\log(n),\\
&\sum^{n}_{t=\xi}\frac{t^2}{n-t}\leq n^2\sum^{n}_{t=\xi}\frac{1}{n-t}\leq C_4n^2\log(n),
\end{align*}
for some positive constants $C_i$ with $1\leq i\leq 4.$ Finally, using the definition $\eta_t=\frac{t\mu}{2}$ and the convexity of $F_n$, we have
\begin{align*}
E^*\left[F_n\left(n^{-1}\sum^{n}_{t=1}\hat{\theta}_{t}\right)-F_n(\hat{\theta}^*)\right]\leq  E^*\left[n^{-1}\sum^{n}_{t=1}F_n(\hat{\theta}_{t})-F_n(\hat{\theta}^*)\right]\leq& \frac{c\log(n)}{n},
\end{align*}
for some $c>0.$ The conclusion thus follows.
\end{proof}

The result below follows from Corollary 3 of \cite{shamir2016without}, which is proved using the transductive learning theory.
\begin{lemma}\label{lem-1}
{\rm 
Under Assumptions A5-A6, we have 
\begin{align*}
E^*\left[\left(\sup_{\theta\in\Theta}\left\{\nabla L_{1:a}(\theta)-\nabla L_{a+1:n}(\theta)\right\}^\top \frac{\theta-\hat{\theta}^*}{\|\theta-\hat{\theta}^*\|}\right)^2\right]\leq C^2_1\left(
\frac{1}{\sqrt{a}}+\frac{1}{\sqrt{n-a}}\right)^2,
\end{align*}
where $C_1$ is some constant that depends on $L_1,L_2$ and $D$ (the diameter of $\Theta$).
}
\end{lemma}

\subsection{Simulation settings}\label{sec:sim-set}
The tables below summarize the values of the regression coefficients $\theta_i$ (for both the logistic and Poisson regressions) within each segment partitioned by the change-point locations $\{\tau_i\}$.

\begin{itemize}
    \item Single change-point ($k=1$): $\tau_1=750$ and
    \begin{center}
\begin{tabular}{  l | l l  l  } 
  \hline
  & $d = 1$ & $d = 3$ & $d = 5$\\
  \hline
  $1\leq i\leq\tau_1$ & $1.2$ & $(1, 1.2, -1)$ & $(1, 1.2, -1, 0.5, -2)$ \\
  $\tau_1<i\leq T$ & 
   $1.2 + \delta_1$ & $ (1, 1.2, -1)+ \delta_3$ & $ (1, 1.2, -1, 0.5, -2)+ \delta_5$\\
  \hline
\end{tabular}
\end{center}

    \item Three change-points ($k=3$): $\tau_1 = 375, \tau_2 = 750, \tau_3 = 1125$ and 
\begin{center}
\begin{tabular}{  l | l | l | l } 
  \hline
   & $d = 1$ & $d = 3$ & $d = 5$\\
  \hline
  $1\leq i\leq \tau_1$ & 
  $1.2$ & $(1, 1.2, -1)$ & $(1, 1.2, -1, 0.5, -2)$ \\
  
  $\tau_1 < i \leq \tau_2$ & 
  $1.2 + \delta_1$ & $ (1, 1.2, -1)+ \delta_3$ & $ (1, 1.2, -1, 0.5, -2)+ \delta_5$\\
  
  $\tau_2 < i \leq \tau_3$ & 
  $1.2$ & $ (1, 1.2, -1)$ & $ (1, 1.2, -1, 0.5, -2)$\\
  
  $\tau_3 < i \leq T$ &
  $1.2 - \delta_1$ & $ (1, 1.2, -1)- \delta_3$ & $ (1, 1.2, -1, 0.5, -2)- \delta_5$\\
  \hline
\end{tabular}
\end{center}

    \item Five change-points ($k=5$): $\tau_1 = 250, \tau_2 = 500, \tau_3 = 750, \tau_4 = 1000, \tau_5 = 1250$ and 
\begin{center}
\begin{tabular}{ l | l | l | l } 
  \hline
   & $d = 1$ & $d = 3$ & $d = 5$\\
  \hline
  $1\leq i\leq \tau_1$ & 
  $1.2$ & $(1, 1.2, -1)$ & $(1, 1.2, -1, 0.5, -2)$ \\
  
  $\tau_1 < i \leq \tau_2$ & 
  $1.2 + \delta_1$ & $ (1, 1.2, -1)+ \delta_3$ & $ (1, 1.2, -1, 0.5, -2)+ \delta_5$\\
  
  $\tau_2 < i \leq \tau_3$ & 
  $1.2$ & $ (1, 1.2, -1)$ & $ (1, 1.2, -1, 0.5, -2)$\\
  
  $\tau_3 < i \leq \tau_4$ &
  $1.2 - \delta_1$ & $ (1, 1.2, -1)- \delta_3$ & $ (1, 1.2, -1, 0.5, -2)- \delta_5$\\
  
  $\tau_4 < i \leq \tau_5$ &
  $1.2$ & $ (1, 1.2, -1)$ & $ (1, 1.2, -1, 0.5, -2)$\\

  $\tau_5 < i \leq T$ &
  $1.2 + \delta_1$ & $ (1, 1.2, -1)+ \delta_3$ & $ (1, 1.2, -1, 0.5, -2)+ \delta_5$\\
  \hline
\end{tabular}
\end{center}
\end{itemize}

\subsection{Multiple epochs}\label{sec:ext}
This section describes an extension of Algorithm \ref{alg1} to allow multiple epochs. Specifically, we will use each data point $K\geq 1$ times in updating the parameter estimates for a particular segment.
The details are summarized in Algorithm \ref{alg2} below. SE with multiple epochs/passes allows more efficient use of the data with an additional computational expense controlled by $K$. We leave a detailed analysis of this trade-off between statistical and computational efficiencies for future investigation.

\begin{algorithm}
\caption{Sequential Updating Algorithm with Multiple Epochs}\label{alg2}
\begin{itemize}
    \item Input the data $\{z_i\}^{T}_{i=1}$, the individual cost function $l(\cdot,\theta)$, the penalty constant $\beta$ and the number of epochs $K$.
    \item Set $F(0)=-\beta$, $\mathcal{C}=\emptyset$ and $R_1=\{0\}$.
    \item Iterate for $t=1,2,\dots,T$:
    \begin{enumerate}
        \item Initialize $S_{t:t}^{(K)}=\hat{\theta}_{t:t}^{(K)}$ and $H_{t:t}^{(K)}$. For $\tau\in R_t\setminus\{t-1\}$, perform the update
        \begin{align*}
        &\hat{\theta}_{\tau+1:t}^{(1,t)}=\mathcal{P}_{\Theta}(\hat{\theta}_{\tau+1:t-1}^{(K)}-H_{\tau+1:t-1}^{(K),-1}\nabla l(z_t,\hat{\theta}_{\tau+1:t-1}^{(K)})),\\
        &H_{\tau+1:t}^{(1,t)}=H_{\tau+1:t-1}^{(K)}+ \mathcal{A}(\hat{\theta}_{\tau+1:t}^{(1,t)}).
        \end{align*}
        Next for $k=2,\dots,K$, perform the update
        \begin{align*}
        &\hat{\theta}_{\tau+1:j}^{(k,t)}=\mathcal{P}_{\Theta}(\hat{\theta}_{\tau+1:j-1}^{(k,t)}-H_{\tau+1:j-1}^{(k,t),-1}\nabla l(z_j,\hat{\theta}_{\tau+1:j-1}^{(k,t)})),\\
        &H_{\tau+1:j}^{(k,t)}=H_{\tau+1:j-1}^{(k,t)}+ \mathcal{A}(\hat{\theta}_{\tau+1:j}^{(k,t)}),
        \end{align*}
        over $j=\tau+1,\dots,t$, where $(\hat{\theta}^{(k,t)}_{\tau+1:\tau},H_{\tau+1:\tau}^{(k,t)})=(\hat{\theta}^{(k-1,t)}_{\tau+1:t}),H_{\tau+1:t}^{(k-1,t)})$.
        Set $\hat{\theta}_{\tau+1:t}^{(K)}=\hat{\theta}_{\tau+1:t}^{(K,t)}$, $H_{\tau+1:t}^{(K)}=H_{\tau+1:t}^{(K,t)}$ and 
        $$S_{\tau+1:t}^{(K)}=S_{\tau+1:t-1}^{(K)}+\hat{\theta}_{\tau+1:t}^{(K)}.$$
        \item For each $\tau\in R_t$, compute
        $$\widehat{C}(\bfz_{\tau+1:t})=\sum^{t}_{i=\tau+1}l\left(z_i,(t-\tau)^{-1}S_{\tau+1:t}^{(K)}\right).$$
        \item Calculate
        \begin{align*}
        & F(t)=\min_{\tau\in R_{t}}\left\{F(\tau)+\widehat{C}(\bfz_{\tau+1:t})+\beta\right\},\\
        & \tau^*=\argmin_{\tau\in R_{t}}\left\{F(\tau)+\widehat{C}(\bfz_{\tau+1:t})+\beta\right\}.
        \end{align*}
        \item Let $\mathcal{C}(t)=\{\mathcal{C}(\tau^*),\tau^*\}$.
        \item Set 
        $$R_{t+1}=\left\{\tau \in R_t\cup\{t\}:F(\tau)+C(\bfz_{\tau+1:t})\leq F(t)\right\}.$$
    \end{enumerate}
       \item Output $\mathcal{C}(T)$.
\end{itemize}
\end{algorithm}

\end{document}